\documentclass[10pt,journal,compsoc]{IEEEtran}
\ifCLASSOPTIONcompsoc
  \usepackage[nocompress]{cite}
\else
  \usepackage{cite}
\fi
\ifCLASSINFOpdf
\else
\fi
\usepackage{caption}
\usepackage{subcaption} 
\usepackage{graphicx}
\usepackage{amsmath}
\graphicspath{ {images/} }
\usepackage{multirow}

\usepackage{wasysym}
\usepackage{pifont}
\usepackage{textcomp}
\usepackage{xcolor}
\usepackage{amssymb,amsfonts}

\usepackage{amsthm}

\usepackage{mathtools} 
\usepackage{alltt, xspace, times, epsfig, url} 
\usepackage{gensymb,xfrac,array,booktabs,bm,xcolor} 
\usepackage[linesnumbered,lined,ruled]{algorithm2e}

\usepackage{tikz}

\newtheorem{theorem}{Theorem}

\newtheorem{lemma}{Lemma}
\newtheorem{definition}{Definition}

\makeatletter
\newcommand{\nosemic}{\renewcommand{\@endalgocfline}{\relax}}
\newcommand{\dosemic}{\renewcommand{\@endalgocfline}{\algocf@endline}}
\let\oldnl\nl
\newcommand{\nonl}{\renewcommand{\nl}{\let\nl\oldnl}}

\begin{document}
%
\title{Securing Distributed SGD against \\ Gradient Leakage Threats}

\author{Wenqi Wei\IEEEauthorrefmark{1}\IEEEauthorrefmark{2}, 
Ling~Liu\IEEEauthorrefmark{2},~\IEEEmembership{Fellow,~IEEE,}  
Jingya Zhou\IEEEauthorrefmark{3}, Ka-Ho Chow\IEEEauthorrefmark{2}, 
and Yanzhao Wu\IEEEauthorrefmark{4}\IEEEauthorrefmark{2}
\IEEEcompsocitemizethanks{
\IEEEcompsocthanksitem \IEEEauthorrefmark{2}School of Computer Science, Georgia Institute of Technology, Atlanta, GA, 30332. 
\IEEEcompsocthanksitem \IEEEauthorrefmark{1}Computer and Information Science Department, Fordham University, New York City, NY, 10023.  
\IEEEcompsocthanksitem \IEEEauthorrefmark{3}School of Computer Science and Technology, Soochow University, Suzhou, Jiangsu, China, 215006. 
\IEEEcompsocthanksitem \IEEEauthorrefmark{4}School of Computing and Information Sciences, Florida International University, 
Miami, FL, 33199.  
\IEEEcompsocthanksitem \indent E-mail: 
wenqiwei@fordham.edu, ling.liu@cc.gatech.edu,  jy\_zhou@suda.\\edu.cn, 
khchow@gatech.edu, yawu@fiu.edu.}
\thanks{Manuscript received xxxx xx, xxxx; revised xxxxx xx, xxxx.}
}

%
%

\markboth{Journal of \LaTeX\ Class Files,~Vol.~14, No.~8, August~2015}%
{Shell \MakeLowercase{\textit{et al.}}: Bare Demo of IEEEtran.cls for Computer Society Journals}
%



\IEEEtitleabstractindextext{%
\begin{abstract}
This paper presents a holistic approach to gradient leakage resilient distributed Stochastic Gradient Descent (SGD). 
{\em First}, we analyze two types of strategies for privacy-enhanced federated learning: (i) gradient pruning with random selection or low-rank filtering and (ii) gradient perturbation with additive random noise or differential privacy noise. We analyze the inherent limitations of these approaches and their underlying impact on privacy guarantee, model accuracy, and attack resilience.
{\em Next}, we present a gradient leakage resilient approach to securing distributed SGD in federated learning, with differential privacy controlled noise as the tool. Unlike conventional methods with the per-client federated noise injection and fixed noise parameter strategy, our approach keeps track of the trend of per-example gradient updates. It  makes adaptive noise injection closely aligned throughout the federated model training.
{\em Finally}, we provide an empirical privacy analysis on the privacy guarantee, model utility, and attack resilience of the proposed approach. 
Extensive evaluation using five benchmark datasets demonstrates that our gradient leakage resilient approach 
can outperform the state-of-the-art methods with competitive accuracy performance, strong differential privacy guarantee, and high resilience against gradient leakage attacks. The code associated with this paper can be found: \url{https://github.com/git-disl/Fed-alphaCDP}.
\end{abstract}

}


\maketitle

\IEEEdisplaynontitleabstractindextext

%
\IEEEpeerreviewmaketitle

\IEEEraisesectionheading{\section{Introduction}\label{sec:intro}}

\IEEEPARstart{F}{ederated} learning is an emerging distributed machine learning paradigm with default privacy. It enables collaborative model training over a large corpus of decentralized data residing on a distributed population of autonomous clients, and all clients can keep their sensitive data private. Clients perform local training with Stochastic Gradient Descent (SGD) or SGD-alike learning and only share local model updates via encrypted communication with the federated server for server aggregation. To deal with potentially unpredictable client availability, only a small subset $K_t$ of the total population of $N$ clients ($K_t \leq N$) is required to contribute to each round $t$ ($1\leq t \leq T$) in federated learning~\cite{mcmahan2017communication}.

\textbf{Privacy Concerns in Federated Learning.} Despite the default privacy by keeping client data local, recent studies reveal that gradient leakage attacks in federated learning may compromise the privacy of client training data \cite{zhu2019deep,zhao2020idlg,geiping2020inverting,wei2020framework,yin2021see}. For example, an unauthorized read on the gradients generated from local SGD can lead to exploitation of per-example gradient updates using the reconstruction-based gradient leakage attack method, resulting in unauthorized disclosure of client private training data. Similarly, peeping at the shared gradient for global aggregation or global-level SGD at the server may give away sensitive information about the training data from a particular client.
Consequently, three fundamental questions are raised in this reality of federated learning: (1) Can a model trained with privacy-enhanced federated learning be scalable for minimizing unauthorized disclosure of sensitive local training data on its clients? (2) What level of resilience can a privacy-preserving federated learning solution provide to protect the privacy of its clients against gradient leakage threats? (3) Can a federated learning model trained with differential privacy guarantee provide high model accuracy and strong gradient leakage resilience at the same time? In this paper, we attempt to answer these questions and argue that an effective differentially private federated learning algorithm can be robust and effective against gradient leakages with high accuracy performance.

In this paper, we first characterize three types of gradient leakage attacks based on where gradients are leaked and the threat model. Then we study gradient compression and gradient perturbation techniques and analyze the inherent limitations of these approaches on gradient leakage resilience. Motivated by the analysis, we present a resilience-enhanced approach to federated learning, using differential privacy controlled noise as the tool. Our method, coined as Fed-$\alpha$CDP, extends the conventional per-client approaches by per-example gradient perturbation and augments with adaptive parameter optimizations. Our approach makes three improvements over existing state-of-the-art solutions. 
{\it First}, 
we inject differential privacy controlled noise to per-example gradients during client's local training such that the local SGD at client is performed on the perturbed gradients. As a result, the local SGD is differentially private. The parameter updates resulting from client's local training and shared with the federated server at each round are perturbed rather than in the original raw format, ensuring that the server-side aggregation is differentially private by composition theorems of differential privacy. 
{\it Second}, in addition to the location of the noise injection,
we revisit the use of the fixed clipping bound that defines the sensitivity in existing differentially private deep learning to address the inherent problem of the resulting constant Gaussian noise variance 
throughout the multi-round local SGD and global aggregation in federated learning.  Given that gradient magnitude demonstrates a decreasing trend as the learning processes, we alternatively define the sensitivity of a differentially private federated learning algorithm using the $l_2$ max of gradients, which results in decreasing Gaussian noise variance during the local SGD at clients' local training. 
{\it Third}, 
we introduce the dynamic decaying noise scale $\sigma$ to enable Gaussian noise variance to follow the trend of gradient updates so that noise variance and noise injection in local SGD are closely aligned. 
We perform an empirical privacy analysis on the proposed Fed-$\alpha$CDP with respect to privacy and utility measured by privacy spending, model accuracy, and gradient leakage resilience. 
With extensive experiments conducted using five benchmark datasets, we evaluate the effectiveness of Fed-$\alpha$CDP in comparison to the existing state-of-the-art approaches from three perspectives: (i) resilience against three types of gradient leakage attacks, (ii) accuracy performance under the same privacy budget, and (iii) privacy guarantee under the same target accuracy  (utility). We show that the proposed solution approach outperforms other alternatives with the best overall performance in terms of gradient leakage resilience, differential privacy guarantee, and model accuracy.


\section{Preliminary}
\label{sec:preliminary}

\subsection{ Differential Privacy: Definitions}
\label{sec:dp-definitions}

\begin{definition}\textbf{Differential privacy~\cite{dwork2014algorithmic}}: Let $\mathcal{D}$ be the domain of possible input data and $\mathcal{R}$ be the range of all possible output. A randomized mechanism $\mathcal{M}$: $\mathcal{D}\rightarrow \mathcal{R}$ satisfies ($\epsilon,\delta$)-differential privacy if for any two input sets $A \subseteq \mathcal{D}$ and $A'\subseteq \mathcal{D}$, differing with only one entry: $||A-A'||_{0}=1$, Equation~\ref{equa:dp} holds with $0 \leq \delta < 1$ and $\epsilon>0$.
\vspace{-0.1cm}
\begin{equation}
\Pr(\mathcal{M}(A) \in \mathcal{R}) \le e^{\epsilon}\Pr(\mathcal{M}(A') \in \mathcal{R}) + \delta.
\label{equa:dp}
\end{equation}
\label{def:dp}
\vspace{-0.7cm}
\end{definition}

A smaller $\epsilon$ would indicate a smaller difference between the output of $\mathcal{M}(A)$ and the output of $\mathcal{M}(A')$. Since $\delta$ is the upper bound probability of $\mathcal{M}(A)$ for breaking $\epsilon$-differential privacy,
a small $\delta$ is desired. Following the literature~\cite{abadi2016deep,yu2019differentially,papernot2018scalable}, $\delta$ is set to $1e-5$ for the rest of the paper. 

\begin{definition}
\textbf{Sensitivity~\cite{dwork2006calibrating}}: Let $\mathcal{D}$ be the domain of possible input and $\mathcal{R}$ be the domain of possible output.  The sensitivity of function $f: \mathcal{D} \rightarrow \mathcal{R}$ is the maximum amount that the function value varies when a single input entry is changed. 
\vspace{-0.1cm}
\begin{equation}
S=\max \nolimits_{A,A' \subseteq \mathcal{D},
||A-A'||_{0}=1} ||f(A)-f(A')||_p.
\label{equa:sensitivity}
\end{equation}
\label{def:sensitivity}
\vspace{-0.7cm}
\end{definition}
\noindent
By Equation~\ref{equa:sensitivity}, if one wants to produce a differentially private algorithm $\mathcal{M}(f)$ by injecting random noise under a specific distribution (e.g., Laplace or Gaussian) to function $f$, one needs to scale the noise to match (bound by) the maximum change defined as the sensitivity of function $f$ with neighboring inputs. 



\begin{theorem}
\textbf{Gaussian mechanism~\cite{dwork2014algorithmic}}: Let $\mathcal{D}$ be the domain of possible input data and $\mathcal{R}$ be the range of all possible output.  With arbitrary privacy parameter $\epsilon \in (0,1)$, applying Gaussian noise $\mathcal{N}(0, \varsigma {^2})$ calibrated to a real-valued function: $f: \mathcal{D}\rightarrow \mathcal{R}$ with noise variance $\varsigma {^2}$ such that $\mathcal{M}(A) = f(A) + \mathcal{N}(0, \varsigma^2)$ is $(\epsilon,\delta)$-differentially private if $\varsigma {^2} > \frac{{2\log (1.25/\delta ) \cdot {S^2}}}{{{\epsilon ^2}}}$. 
\label{theorem:gaussian}
\end{theorem}

\begin{lemma}
Let $\varsigma^2$ in Gaussian mechanism be $\sigma^2S^2$ where $\sigma$ is the noise scale, and $S$ is the $l_2$ sensitivity. We have the noise scale $\sigma$ satisfying $\sigma^2>\frac{{2\log (1.25/\delta )}}{{{\epsilon ^2}}}.$
\label{lemma:gaussian_mechanism}
\end{lemma}

Lemma~\ref{lemma:gaussian_mechanism} is straightforward with Theorem~\ref{theorem:gaussian}. 
When $\delta$ is given and fixed, the noise scale $\sigma$ has an inverse correlation with privacy loss $\epsilon$, implying that one can calculate the smallest $\epsilon$ differential privacy guarantee of the function based on the noise scale. 

{\bf Privacy parameters in baseline DPSGD implementation.\/} 
By Definition~\ref{def:sensitivity}, the sensitivity of a differentially private function is defined as the maximum amount that the function value varies when a single input entry is changed. Hence, the sensitivity of a differentially private function may vary for different input batches during different iterations of deep learning~\cite{wei2021gradient_tifs}. For federated learning, the sensitivity of the local model can be different for different local iterations at different clients and rounds.
Based on Theorem~\ref{theorem:gaussian} and Lemma~\ref{lemma:gaussian_mechanism}, the Gaussian noise for a differentially private function is calibrated with noise variance defined by noise scale $\sigma$ and sensitivity $S$ of the function. 
The baseline DPSGD implementation~\cite{abadi2016deep}, followed by most of the work~\cite{yu2019differentially,mcmahan2017learning,wei2020federated,wei2021gradient} suggests using a fixed clipping parameter $C$ to estimate sensitivity $S$, and a fixed noise scale $\sigma$.


\subsection{Properties of Differential Privacy}

Several important properties of differential privacy are essential in tracking privacy spending. They designate the privacy composition for a sequence of differential privacy mechanisms on the same dataset, for mechanisms that run in parallel over disjoint datasets, and for the post-processing of a differentially private mechanism.

\begin{theorem}
\textbf{Composition theorem~\cite{dwork2014algorithmic}}:  Let $\mathcal{M}_i: \mathcal{D}\rightarrow \mathcal{R}_i$ be a randomized function that is $(\epsilon_i,\delta_i)$-differentially private. If $\mathcal{M}$ is a sequence of consecutive invocations (executions) 
of ($\epsilon_i,\delta_i$)-differentially private algorithm $\mathcal{M}_i$, then $\mathcal{M}$ is ($\sum_i \epsilon_i, \sum_i \delta_i$)-differentially private.
\label{theorem:composition_sequential}
\end{theorem}

\begin{theorem}
\textbf{Parallel Composition~\cite{mcsherry2009privacy}}:  Let $\mathcal{M}_i: \mathcal{D}_i\rightarrow \mathcal{R}_i$ be a randomized function that is $(\epsilon_i,\delta_i)$-differentially private on a disjointed subset of the data.
If $\mathcal{M}$ consists of a set of ($\epsilon_i,\delta_i$)-differentially private algorithms that are invoked and executed in parallel, then $\mathcal{M}$ is ($\max_i \epsilon_i, \max_i \delta_i$)-differentially private.
\label{theorem:composition_parallel}
\end{theorem}


\begin{theorem}
\textbf{Post Processing~\cite{dwork2014algorithmic}}: Let $\mathcal{D}$ be the domain of possible input data and $\mathcal{R}$ be the range of all possible output. Let $\mathcal{M}: \mathcal{D}\rightarrow \mathcal{R}$ be a randomized function that is $(\epsilon,\delta)$-differentially private. Let $h:\mathcal{R}\rightarrow \mathcal{R'}$  be an arbitrarily randomized or deterministic mapping. Then $h \circ \mathcal{M}: \mathcal{D}\rightarrow \mathcal{R'}$ is $(\epsilon,\delta)$-differentially private.
\label{theorem:post_processing}
\end{theorem}

Theorem~\ref{theorem:composition_sequential} states that a randomized function $\mathcal{M}$ that consists of a sequence of $n$ differentially private mechanisms is differentially private. 
Theorem~\ref{theorem:composition_parallel} states that one can view the parallel composition of a set of $n$ differentially private mechanisms executed in parallel as the case in which each differentially private mechanism is applied to one of the $n$ disjointed subsets of the input dataset concurrently. The overall privacy guarantee of this parallel composition will be defined by the maximum privacy loss among the $n$ individual privacy losses. 
Theorem~\ref{theorem:post_processing} states another important property of the differentially private algorithm: any post-processing applied to the result of a $(\epsilon,\delta)$-differentially private mechanism is also $(\epsilon,\delta)$-differential private with the same privacy guarantee. 
The above differential privacy definitions and theorems form the essential foundation for developing differentially private federated learning approaches. Besides composition, the data sampling method also contributes to the privacy spending accumulation via privacy amplification~\cite{balle2018privacy}.


\section{Gradient Leakage Threats}

\subsection{Threat Model}
\label{sec:threatmodel}
\noindent


Assumptions. {\em On the data side}, we assume that data at rest and data in network transit are encrypted and secure. The main attack surface is during data-in-use. We focus on gradient leakage induced threats to client training data in the presence of malicious or semi-curious adversary. Hence, we focus on unauthorized access to gradient data during local training on a client and the global SGD for aggregation of local model updates performed at the federated server. Hence, we assume that the attackers cannot gain access to the training data prior to feeding the decrypted training data to the DNN algorithm during local training.  

{\em On attacker side}, we assume that the adversary on a compromised client {\em only} performs gradient leakage attack, which seeks unauthorized access to gradients (training parameter updates) of participating clients at three possible attack surfaces: (i) at the federated server, (ii) at the client after local training prior to sending encrypted gradients to the federated server, and (iii) during local training prior to local SGD performed at each client. The {\bf type-0 attack} occurs at the federated server. We assume that the server is compromised and the adversary can access the shared per-client gradient prior to performing the global SGD at the federated server. However, the adversary cannot obtain the per-example gradient or the accumulated gradients during per iteration of the local training process prior to performing local SGD at each client. The {\bf type-1 attack} occurs at the client but the adversary on a client can only gain access to the per-client local gradient after the local training is completed and prior to the client encrypting and sending the local training gradient  to the federated server at each global round. Finally, the {\bf type-2 attack}  corresponds to the third type of attack surface where the adversary at client may gain access to the per-example gradient prior to performing local SGD at each iteration during local training. This type-2 attack can result in a high attack success rate by reconstructing high-quality sensitive training data.

 \begin{figure}[t]
\centerline{\includegraphics[scale=.60]{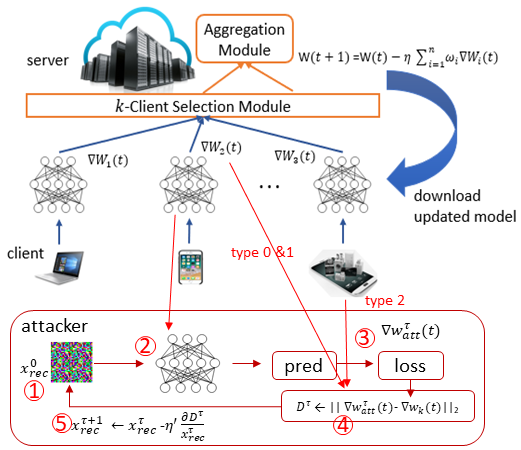}}
\vspace{-0.1cm}
\caption{\small Attack schema.}
\label{fig:attack_schema}
\vspace{-0.2cm}
\end{figure}

\subsection{Attack Procedure}

\textbf{Figure~\ref{fig:attack_schema}} gives a sketch of the gradient leakage attack algorithm, which configures and executes the reconstruction attack in five steps: 
(1) the attacker configures the initialization seed ($x^{0}_{rec}(t)$), a dummy data of the same resolution (or attribute structure for text) as the training data. \cite{wei2020framework} showed some significant impact of 
different initialization seeds on the attack success rate and attack cost ($\#$ attack iterations to succeed). (2) The dummy attack seed is fed into the client's local model. (3) The gradient of the dummy attack seed is obtained by backpropagation. (4) The gradient loss is computed using a vector distance loss function, e.g., $L_2$, between the gradient of the attack seed and the actual gradient from the client's local training. The choice of this reconstruction loss function is another tunable attack parameter. (5) The dummy attack seed is modified by the attack reconstruction learning algorithm. It aims to minimize the vector distance loss by a loss optimizer such that the gradients of the reconstructed seed $x^{i}_{rec}(t)$ at round $i$ will be closer to the actual gradient updates stolen from the client upon the completion (type 0 \& 1) or during the local training (type 2). This attack reconstruction iterates until it reaches the attack termination condition ($\tau$), typically defined by the $\#$rounds, e.g., 300 (also a configurable attack parameter). If the reconstruction loss is smaller than the specified distance threshold, then the attack is successful. \textbf{Figure~\ref{fig:attack_vis}} provides a visualization by attack examples of CPL~\cite{wei2020framework} on MNIST, CIFAR10, and LFW, respectively. The type-0 \& 1 gradient leakage attack is performed using the LFW dataset on the batched gradients with batch size 3. 
The details of these datasets as well as their setting are provided in Section~\ref{sec:experiments}.


     \begin{figure}[t]
\centerline{\includegraphics[scale=.40]{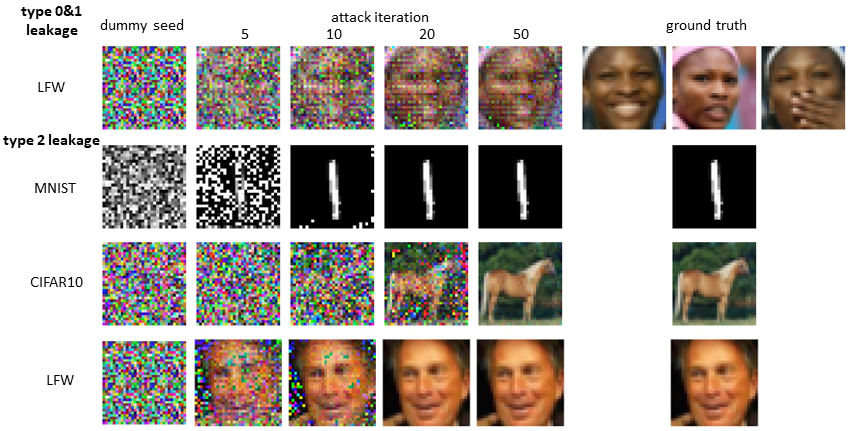}}
\vspace{-0.1cm}
\caption{\small Attack visualization.}
\label{fig:attack_vis}
\vspace{-0.2cm}
\end{figure}

\begin{table}[t]
\centering
\caption{\small Comparison of the representative gradient leakage attacks.}
\vspace{-0.1cm}
\scalebox{0.90}{
\small{
\begin{tabular}{|cc|c|c|c|}
\hline
\multicolumn{2}{|c|}{}                                    & MNIST &  CIFAR10 & LFW   \\ \hline
\multicolumn{1}{|c|}{\multirow{3}{*}{DLG~\cite{zhu2019deep}}}     & ASR      & 0.686         & 0.754   & 0.857 \\ \cline{2-5} 
\multicolumn{1}{|c|}{}                         & quality  & 0.145         & 0.109   & 0.089 \\ \cline{2-5} 
\multicolumn{1}{|c|}{}                         & reconstruction iteration & 18.4          & 114.5   & 69.2  \\ \hline
\multicolumn{1}{|c|}{\multirow{3}{*}{GradInv~\cite{geiping2020inverting}}} & ASR      & 1               & 0.985   & 0.994 \\ \cline{2-5} 
\multicolumn{1}{|c|}{}                         & quality  & 0.122         & 0.095   & 0.071 \\ \cline{2-5} 
\multicolumn{1}{|c|}{}                         & reconstruction iteration & 846          & 2135    & 1826  \\ \hline
\multicolumn{1}{|c|}{\multirow{3}{*}{CPL~\cite{wei2020framework}}}     & ASR      & 1                 & 0.973   & 1     \\ \cline{2-5} 
\multicolumn{1}{|c|}{}                         & quality  & 0.120         & 0.094   & 0.067 \\ \cline{2-5} 
\multicolumn{1}{|c|}{}                         & reconstruction iteration & 11.5        & 28.3    & 25    \\ \hline
\end{tabular}
}}
\label{table:attack_compare}
 \vspace{-0.4cm}
\end{table} 

\textbf{Table~\ref{table:attack_compare}} compares the type-2 leakage under three representative gradient leakage attacks, in terms of $l_2$ leakage the attack success rate (ASR), reconstruction quality measured by RMSE, 
and reconstruction iterations. 
 For the rest of the paper, we adopt the
most effective and efficient attack with patterned-initialization at the first iteration (epoch) of federated learning by following~\cite{wei2020framework}.  



\section{Analysis on Gradient Transformation}
\label{sec:existing_limitation}

In this section, we examine gradient compression and gradient perturbation, which are methods in the literature applicable to gradient transformation in federated learning. We analyze the resilience each may provide against gradient leakage attacks. 

\subsection{Gradient Compression}
Gradient pruning is a common technique for DNN model compression. In each round $t$ of federated learning, each participating client sends a full vector of local training parameter update to the federated server. This step can be the communication bottleneck for large models on complex data. Several communication-efficient federated learning protocols have been proposed by employing structured or sketched updates~\cite{mcmahan2017communication,konevcny2016federated}. 
The former directly learns an update from a pre-specified structure, such as a low-rank matrix and random masks. The latter compresses the learned full vector of model parameter update to ensure a high compression ratio with a low-value loss before sending it to the server. We below describe two simple pruning approaches for utility-aware gradient compression. 

 \begin{figure*}
\centerline{\includegraphics[scale=.42]{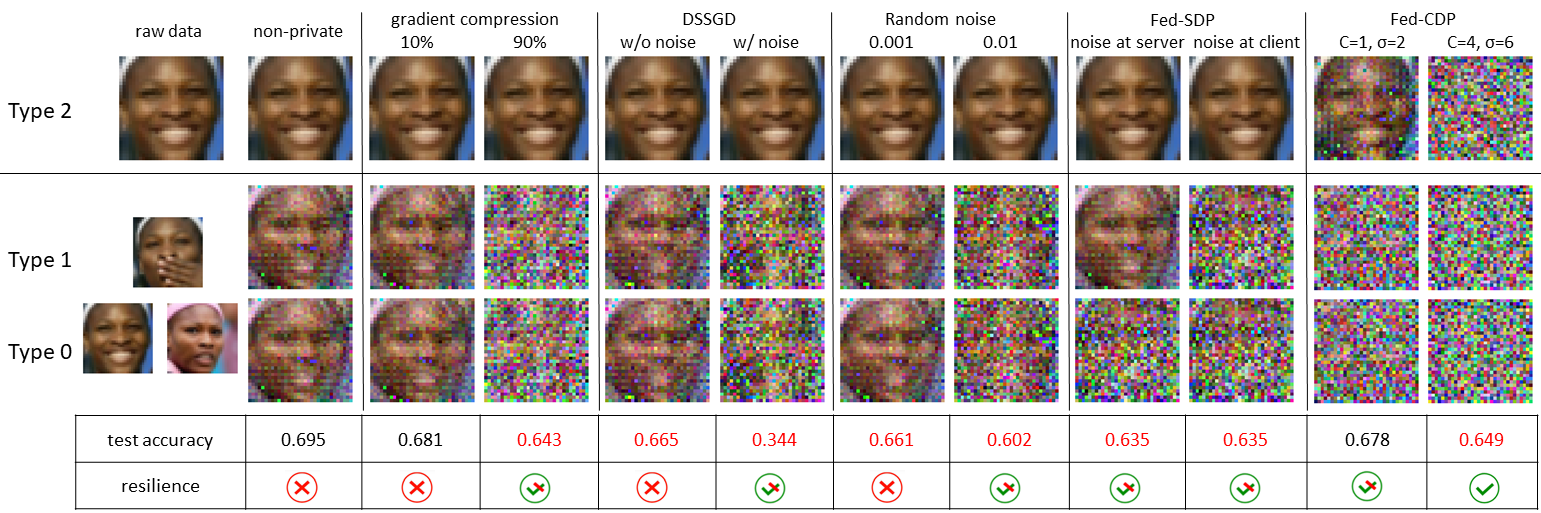}}
\vspace{-0.1cm}
\caption{\small Limitations of existing gradient transformation approaches on gradient privacy leakage.}
\label{table:limitation}
\vspace{-0.4cm}
\end{figure*}

\textbf{Threshold-based gradient pruning.} 
Threshold-based gradient pruning~\cite{lin2018deep} sends only the important gradients to the federated server at each round, i.e., the gradient coordinates whose magnitude is larger than a threshold. The gradients to be shared are first being sorted, and then a certain percentage of very small but high precision gradients are replaced with zero, based on a default pruning threshold $\mu\%$. The level of resilience under gradient pruning depends on both the dataset and the pruning threshold $\mu\%$.
A study in~\cite{wei2020framework} reveals that the resilience of gradient pruning only takes effect once more than 40\% of gradients are pruned for MNIST, 50\% for CIFAR10, and 80\% for LFW. This implies that (i) the gradients transformed by low-rank pruning may still intrude on client privacy under gradient leakage attacks if the gradient compression ratio is not sufficiently large, and (ii) it is hard to set a universal compression ratio that works for all datasets.


\textbf{Random pruning with Distributed Selected SGD.} 
Alternatively, a client may randomly select a subset of gradient coordinates whose values are above a given threshold and share this randomized subset with the federated server (DSSGD~\cite{shokri2015privacy}). However,
DSSGD is at most on par with the value-based pruning on model accuracy and leakage resilience. 
To mask the actual value of the gradient, another option of DSSGD is to add differential privacy noise to the selected subset of gradients before sharing them with the server. 
However, using the differential privacy controlled noise may result in a large accuracy deficit as reported in~\cite{shokri2015privacy}. Moreover, due to the coordinate-level gradient noise, the globally trained model does not offer instance-level
differential privacy protection despite the differential
privacy noise added at each round.

\subsection{Gradient Perturbation}


\textbf{Random noise addition.} It is clear that a larger injected noise will alter the raw gradients more but may also hurt the model accuracy of federated learning. The open challenge for random noise injection is how much noise is sufficient for leakage resilience while maintaining good model accuracy. 
Studies in~\cite{wei2020framework,zhu2019deep} have shown that when the noise variance is high enough, e.g., set to 0.01, for gradient perturbation, the gradient leakage induced violation of client privacy can be prevented at the cost of a significant accuracy loss of 10\% for LFW and 30\% for CIFAR100. Such an unstable effect, compounded with the potential accuracy drop, makes the additive random noise difficult to scale.

{\bf Client-level differential privacy noise.}
Most existing approaches to differentially private federated learning~\cite{mcmahan2017learning,geyer2017differentially} only make client-level protection by making server-side global-level aggregation differentially private. This is done by noise injection on the client update through the iterative $T$ rounds of training, regardless of whether the differential privacy noise is injected at the federated server before performing global SGD~\cite{mcmahan2017learning} or at each of the $K_t$ clients before encrypting the gradients and sending them to the federated server~\cite{geyer2017differentially}. Therefore, we coin this approach as Fed-SDP. By the definition and composition theorems of differential privacy, the trained global model in Fed-SDP only provides differential privacy at the client level, which ensures that an adversary cannot tell whether a local training parameter update from a single client was used, even if the adversary can observe the output of the global model $\mathcal{F}$. However, the per-example gradients used in the local training are not perturbed, and the local SGD is not differentially private. Hence, Fed-SDP cannot provide differential privacy guarantee at the instance level. 

Consider the three types of gradient leakage threats and Fed-SDP performs clipping and noise injection prior to performing aggregation using server-side SGD, if the differential privacy noise is added at the federated server like in~\cite{mcmahan2017learning}, then it is only effective against the type 0 leakage but still vulnerable to type 1 and type 2 gradient leakage attacks, even with the client-level differential privacy guarantee. Suppose Fed-SDP performs noise injection at each of the $K_t$ contributing clients for each round $t$ like in~\cite{geyer2017differentially}. In that case, it can combat both type 0 and type 1 leakages but remain vulnerable to the type 2 gradient leakage attacks. This is because type 2 leakage occurs prior to performing local SGD per iteration during local training. Similarly for both gradient pruning methods and the additive random noise solution, as they perform gradient transformation after the client's local training, the client's local training is unprotected and is vulnerable to type-2 leakage.

{\bf Instance-level differential privacy noise.}
By comparison, \cite{wei2021gradient} introduces per-example noise to federated learning as Fed-CDP and ensures that client-side local SGD during the client's local training is differentially private. Therefore, Fed-CDP provides per-example differential privacy, such that even if an adversary can observe the output of the global model trained by Fed-CDP, he cannot
tell whether the user information defined by a single
example was used in training. Due to Composition theorems and Post Processing, the per-client local update is sanitized with the compound noise accumulated from multiple local iterations. Thus, the subsequent server aggregation is differentially private. Accordingly, Fed-CDP offers both instance-level differential privacy and client-level differential privacy.

Both Fed-SDP~\cite{mcmahan2017learning,geyer2017differentially,wei2020federated} and Fed-CDP~\cite{wei2021gradient} commonly use a pre-defined fixed clipping bound $C$ and the constant clipping method to cap the gradients prior to perturbing the gradients with differential privacy controlled noise, and all follow the DPSGD implementation in~\cite{abadi2016deep}, which utilizes the fixed clipping parameter $C$ as an estimate of the sensitivity, resulting in using a fixed sensitivity in all rounds of federated learning. By Lemma~\ref{lemma:gaussian_mechanism}, the use of fixed $S$ and $\sigma$ will result in constant noise variance, and thus a fixed amount of noise is added for gradient perturbation each time over all $T$ rounds of federated learning.  {\bf Figure~\ref{table:limitation}} provide a visual illustration of the limitation of these methods on test accuracy and leakage resilience under three types of gradient leakage attacks for LFW. Even for Fed-CDP, we find that accuracy-oriented privacy parameter search cannot always guarantee gradient leakage resilience when under fixed privacy parameters (second column to the right in Figure~\ref{table:limitation}). The resilience-accuracy co-design of Fed-CDP could be at the cost of accuracy (last column in Figure~\ref{table:limitation}).
Accordingly, we design Fed-$\alpha$CDP, a new approach to federated learning with differential privacy with three enhancements on how differential privacy noise is computed, injected, and controlled. Our approach aims to provide high resilience against all three types of gradient leakages, strong differential privacy guarantee, and high model accuracy.



 

\section{Fed-$\alpha$CDP: a three-in-one solution}
\label{sec:adaptive}

In this section, we describe Fed-$\alpha$CDP, including the design considerations and the three key enhancements that distinguish Fed-$\alpha$CDP approach to the conventional differentially private federated learning: (i) enforcing client-side instance-level differential privacy, (ii) revisiting and improving sensitivity parameter $S$ with tighter estimation, and (iii) revisiting fixed noise scale parameter $\sigma$ with dynamic decaying strategies.


\subsection{Fed-$\alpha$CDP: Per-Example Perturbation at Client}

To ensure that the global model $\mathcal{F}$ jointly trained over $T$ rounds of federated learning satisfies $(\epsilon,\delta)$ differential privacy, we need to ensure the server-side SGD algorithm, denoted by $F_t$, to satisfy $(\epsilon_t,\delta_t)$-differential privacy, where $\epsilon=\sum_{t=1}^T \epsilon_{t}$ and $\delta=\sum_{t=1}^T \delta_{t}$. There are two ways to make $F_t$ differentially private. 
The first approach is to add $\epsilon_t$-controlled noise to the local gradient update obtained after the completion of local training at each client $i$ for round $t$ and all $K_t$ clients ($1\leq i\leq K_t$) as in Fed-SDP.
The second approach is to add differential privacy noise to the per-example gradient prior to performing the local SGD over each batch of local training data $D_i$. Given that the local training at client $i$ is performed over $L$ iterations ($1\leq l\leq L$) on $D_i$ for round $t$ ($1\leq i\leq K_t$, $1\leq t\leq T$), we need to ensure the local SGD algorithm performed at each iteration $l$, denoted by $g_{til}$, to satisfy $(\epsilon_{tl},\delta_{tl})$-differential privacy. By parallel composition over $K_t$ clients, we have $\epsilon_t =\max_{i \in K_t} \epsilon_{ti}$ and $\delta_t = \max_{i\in K_t} \delta_{ti}$, and by sequential composition over the $L$ iterations of local training on client $i$ for round $t$, we have $\epsilon_{ti}=\sum_{l=1}^L \epsilon_{til}$ and $\delta_{ti}=\sum_{l=1}^T \delta_{til}$ by sequential composition. 


{\bf Sanitizing per-example gradients. \/}  Based on the above analysis, Fed-$\alpha$CDP injects differential privacy controlled random noise prior to performing local SGD for every local iteration $l$ ($1\leq l \leq L$) at each contributing client $i$ ($1\leq i\leq K_t$) in each round $t$ ($1\leq t\leq T$). This ensures that the local SGD algorithm $g_{til}$ satisfies $(\epsilon_{til}, \delta_{til})$-differential privacy. By Theorem~\ref{theorem:gaussian} and Lemma~\ref{lemma:gaussian_mechanism}, the Gaussian noise $\mathcal{N}(0,\sigma^2S^2)$ will be added to the per-example gradient in each local iteration prior to performing local SGD algorithm ($g_{til}$) over the batch $B$ of training examples. 
For a DNN of $M$ layers, we add this noise to per-example gradient layer by layer after performing gradient clipping with a pre-defined clipping bound $C$. 


{\bf Constant clipping method with clipping bound $C$.\/} First, $g_{til}$ draws a batch of $B$ samples from $D_i$, and for each example $j$ in the batch, $g_{til}$ computes the layer-wise per-example gradient vector, $\nabla W_{ij}(t)_{lm}$ ($1\leq m\leq M$), by minimizing the empirical loss function. Next, $g_{til}$ computes the $l_2$ norm of per-example gradient for each of the $M$ layers and performs clipping on the per-example gradients. If the $l_2$ norm of per-example gradient for layer $m$ is no larger than the fixed clipping bound, i.e., $||\nabla W_{ij}(t)_{lm}||_2 \leq C$, then the gradient value $\nabla W_{ij}(t)_{lm}$ is preserved. Otherwise $||\nabla W_{ij}(t)_{lm}||_2>C$, we clip the per-example gradient vector $\nabla W_{ij}(t)_{lm}$ for each layer $m$ so that its $l_2$ norm is reduced to $C$. We denote the clipped per-example gradient as $\overline \nabla W_{ij}(t)_{lm}$ for layer $m$ of local training example $j$ in local iteration $l$.

\subsection{Fed-$\alpha$CDP: Revisiting Sensitivity $S$}


Due to the nature of gradient descent,  the trend of $l_2$ norm of gradients as the training progresses would decrease. This implies that using the fixed clipping bound $C$ to define sensitivity $S$ can be a very loose approximation of the actual $l_2$ sensitivity by Definition~\ref{def:sensitivity}, especially at the later stage of training near convergence. With a fixed sensitivity $S$ and noise scale $\sigma$, the Gaussian noise with variance $\mathcal{N}(0,\sigma^2S^2)$ will result in injecting a fixed amount of differential privacy noise over the course of iterative federated learning. In most rounds, especially later rounds of training, such excessive noise injection is unnecessary and can incur adverse effects on both accuracy and convergence. Meanwhile, regarding gradient leakage attacks, per-client local gradient update at early rounds is much more informative compared to that of later rounds~\cite{wei2020framework}. Therefore, it will be more effective if a larger noise at early rounds and a smaller noise at later rounds of federated learning is injected, implying smaller privacy loss ($\epsilon$) allocates to early rounds and larger privacy spending for later rounds instead of uniform and fixed privacy spending allocation over $T$ rounds. Motivated by these limitations of fixed clipping based sensitivity, we hereby revisit the sensitivity definition and adapt it to federated learning so that the sensitivity strictly aligns to the gradient $l_2$ norm and keep track of the sensitivity of the local training model. 
Specifically, we promote to use the max $l_2$ norm of the per-example gradient in a batch as the sensitivity in Fed-$\alpha$CDP for two reasons. First, the sensitivity of the local SGD function $g_{til}$ differs for different local iterations at the same client $i$ and for the same round $t$ because different iterations correspond to different batches of training examples randomly sampled from the local training dataset $D_i$. Second, the $l_2$-max computed after clipping reflects more accurately the actual sensitivity of $g_{til}$ by following the sensitivity Definition~\ref{def:sensitivity}. 
Consider two scenarios: (i) When the $l_2$ norm of all per-example gradients in a batch is smaller than the pre-defined clipping bound $C$, then the clipping bound $C$ is undesirably a loose estimation of the sensitivity of $g_{til}$ under any given $l$, $i$, and $t$. For fixed $l$, fixed $i$, and fixed $t$, the max $l_2$ norm among the per-example gradients over the entire batch for iteration $l$ is, in fact, a tight estimation of sensitivity for noise injection.  If we instead define the sensitivity of $g_{til}$ by the max $l_2$ norm among these per-example gradients in the batch, we will correct the problems in the above scenario. (ii) When any of the per-example gradients in a batch is larger than the clipping bound, the sensitivity of $g_{til}$ is set to $C$.  In summary, the $l_2$-max sensitivity will take whichever is smaller of the max $l_2$ norm and the clipping bound $C$.  
In Fed-$\alpha$CDP, we can perform clipping layer by layer and example by example, but we need to finish computing the $l_2$ norm of all $M$ layers for all examples in the batch of iteration $l$ in order to obtain the max $l_2$ norm as the tight estimation of the actual sensitivity of $g_{til}$.  
This $l_2$-max sensitivity definition is dependent on the local training function $g_{til}$, which tends to vary for different $l$ ($1\leq l\leq L$), different client $i$ ($1\leq i\leq K_t$), and different round $t$ ($1\leq t\leq T$). Even with given client $i$ and given round $t$, the randomly sampled batch over $D_i$ will be different for different iteration $l$. Furthermore, the function $g_{til}$ will produce different outputs for different clients at different rounds. Hence, this $l_2$-max sensitivity is adaptive with respect to every local iteration, every client, and every round.

{\bf Gradient perturbation with $l_2$-max sensitivity.\/}
After clipping, $g_{til}$ will inject $(\epsilon_{til}, \delta_{til})$-differential privacy controlled Gaussian noise $\mathcal{N}(0,\sigma^2S^2)$ to the layer-wise clipped per-example gradients in the batch: $\widetilde \nabla W_{i}(t)_{l} = \overline \nabla W_{i}(t)_{l} + \mathcal{N}(0,\sigma^2S^2)$, before performing batch averaging over noisy gradients: $\widetilde \nabla W_{i}(t)_{l}=\frac{1}{B}\sum \nolimits_j^B \widetilde \nabla W_{ij}(t)_{l}$. Then the local training function $g_{til}$ performs the gradient descent for iteration $l$: $W_i(t)_{l}=W_i(t)_{l-1}-\eta \widetilde \nabla {W_i(t)_{l-1}}$. The same process of gradient sanitization repeats for each iteration $l$ until all $L$ local iterations are completed. Now the differentially private function $f_{ti}$ is generated by a sequential composition of $L$ invocations of $g_{til}$ ($1 \leq l\leq L$), which produces the local model update to be shared by client $i$ to the federated server for round $t$. 

\subsection{Fed-$\alpha$CDP[$\sigma$]: Dynamic Decaying Noise Scale}
\label{sec:adaptiveSigma}

With the dynamic $l_2$-max sensitivity, Fed-$\alpha$CDP will inject a larger noise at early rounds and a smaller noise at later rounds because the descending trend of $l_2$-max sensitivity results in the declining Gaussian variance as the training progresses for a fixed noise scale $\sigma$. Based on Theorem~\ref{theorem:gaussian} and Lemma~\ref{lemma:gaussian_mechanism}, we can facilitate the decreasing trend for the Gaussian noise variance with a dynamically decaying noise scale instead of a pre-defined fixed $\sigma$. This can be implemented using a smooth decay function over the number of rounds in federated learning. We use Fed-$\alpha$CDP[$\sigma$] when incorporating a dynamic noise scale.


\textbf{Noise scale decay functions.} 
Inspired by dynamic learning rate polices~\cite{wu2019demystifying}, we consider four adaptive policies for implementing adaptive $\sigma$ decay with a starting $\sigma_0$. Each of the policies will progressively decrease the noise scale $\sigma$ as the number of rounds for federated learning increases. 

\textit{Linear decay}: The $\sigma$ decays linearly, as $\sigma_t=\sigma_0(1-\gamma_1 t)$ where $\gamma_1>0$ is the smooth controlling term for $\sigma$ at round $t$.

\textit{Staircase decay}: The $\sigma$ decays with a step function: $\sigma_t=\sigma_0(1-\gamma_2 \lfloor t/\Gamma \rfloor )$ where $\Gamma$ is the step size, or the changing interval for $\sigma_t$ allocation and $\gamma_2$ controls the magnitude of staircase.

\textit{Exponential decay}: The $\sigma$ decays by an exponential function: $\sigma_t=\sigma_0 e^{-\gamma_3 t}$ where $\gamma_3$ is the exponential controlling term.

\textit{Cyclic decay}: The $\sigma$ decay follows a linear cyclic Cosine annealing policy inspired by \cite{huang2017snapshot}, formulated as $
{\sigma_{t}} = \frac{{{\sigma _{0}}}}{2}\left( {\cos \left( {\frac{{\pi \bmod (t-1,\lceil T/\gamma_4 \rceil )}}{\lceil T/\gamma_4 \rceil} } \right) + 1} \right)$ where $\gamma_4$ denotes the number of cycles and the $\sigma_{t+mod(t-1,\lceil T/\gamma_4 \rceil )) }$ will restart from $\sigma_0$ at the beginning of every $\lceil T/\gamma_4 \rceil$ rounds. 

With a fixed $\sigma$ and fixed $S$, $\epsilon$ is typically distributed uniformly over the $T$ rounds. Under the same privacy budget and the fixed noise variance, the same amount of noise is injected each time regardless of the trend of the $l_2$ norm of the gradients. 
In comparison, when using the dynamically decaying $\sigma$, large differential privacy controlled noise will be injected at the early stage of training. With the noise scale gradually decaying,  the amount of differential privacy noise injected will be steadily reduced as the number of training rounds increases, resulting in smaller noise in the later rounds of federated learning. By Lemma~\ref{lemma:gaussian_mechanism}, $\sigma$ is anti-correlated to $\epsilon$. The dynamic decay of $\sigma$ will result in non-uniform distribution of $\epsilon$ over the $T$ training rounds: with small $\epsilon$ indicating small privacy loss by enforcing a larger noise scale. The smaller privacy spending will result in injecting larger noise in early rounds to protect informative gradients against leakages. Then by employing a smaller noise scale and accordingly smaller noise variance and increased $\epsilon$ spending in the later stage of federated learning, it results in faster convergence and higher accuracy performance under a given privacy budget.


While we want to construct dynamic differential privacy noise, determining $\sigma_t$ will need to take the following three factors into consideration: (1) the starting $\sigma_0$ need to be large enough to prevent gradient leakages. Note that general accuracy-driven privacy parameter search cannot always guarantee gradient leakage resilience. Therefore, we select the privacy parameter settings proven empirically to be resilient~\cite{wei2021gradient}: $\sigma_0 \geq \frac{C*\sigma}{S_{dyn}}$ for the initial setting. (2) the ending $\sigma_T$ cannot be too small as otherwise the $\epsilon$ privacy spending would explode, resulting in a poor DP protection.
(3) the amount of noise injected is yet not too much to affect the desired accuracy performance of the global model.

\begin{algorithm}[t]
\footnotesize
\caption{\footnotesize Fed-$\alpha$CDP[$\sigma$]}\label{client_dpsgd_12252020-l2max}
\KwIn{\# total clients $N$, \# clients per round $K_t$, maximum global round $T$, noise scale \textcolor{blue}{$\sigma_t$}, clipping bound $C$.}
\textbf{Server initialization and broadcast:} global model $W(0)$, local batch size $B$, \# local iteration $L$, local learning rate $\eta$.  \\ 
\For{round $t$ in $\{0,1,2\dots, T-1\}$}{
 \For{client $i$ in $\{1,2,\dots, K_t\}$}{
\nonl  \textbf{// download global model $W(t)$.}
	\vspace{0.2cm}
    
    \textbf{client $i$ do:   // start $f_{ti}$ on $D_i$.} \\
    $W_i(t)_0 \leftarrow W(t)$ \\
    \For{local iteration $l$ in $\{0,1,2...L-1\}$}{
    \nonl  \textbf{// batch processing, sampled over $D_i$ , start $g_{til}$.} \\
          \For{instance $j$ in $\{1,...B\}$ }{ 
         \For{layer $m$ in $\{1,...M\}$ }{ 
      \nonl  \textbf{// compute per-example gradient $\nabla W_{ij}(t)_l$ for layer $m$ of sample $j$ in the batch at local iteration $l$ from client $i$.} \\
        $\nabla W_{ij}(t)_l\leftarrow \{\nabla W_{ij}(t)_{lm}\}$ \\
      \nonl    \textbf{// compute $l_2$ norm for layer $m$ of sample $j$ at local iteration $l$.} \\
        $l2_{ij}(t)_{lm} = ||\nabla  W_{ij}(t)_{lm}||_2$ \\
      \nonl   \textbf{// clip per-example gradients by coordinate for layer $m$ of sample $j$ at local iteration $l$.}  \\
        $\overline \nabla W_{ij}(t)_{lm} \leftarrow \nabla W_{ij}(t)_{lm} * \min \{1, \frac{C}{l2_{ij}(t)_{lm}}\}$  \label{algo:clippingline} 
         } 
      \nonl   \textbf{// obtain the clipped per-example gradients for sample $j$ at local iteration $l$.}  \\
         $\overline \nabla W_{ij}(t)_l \leftarrow \{\overline \nabla W_{ij}(t)_{lm}\} , m=1,\dots, M$ \\
    }
\textbf{// compute batch gradients on $M$ layers for iteration $l$:} \\
         $ \overline \nabla W_i(t)_{l} \leftarrow \frac{1}{B}\sum \nolimits_{j=1}^B \overline \nabla W_{ij}(t)_l$ \\  
              \nonl        \textbf{// compute the max of $l_2$ norm over $M$ layers on the gradients of batch instance \{1,...B\} for local iteration $l$, assign the sensitivity $S$. } \\
        \textcolor{blue}{$S \leftarrow \max_{i,m} ||\overline  \nabla  W_{i}(t)_{lm}||_2, i=1,...,B, m=1,...,M$} \\ 
\nonl   \textbf{// compute sanitized batch gradients.} \\
         $  \widetilde \nabla W_i(t)_{l} \leftarrow \overline \nabla W_{i}(t)_l + \mathcal{N}(0, \textcolor{blue}{\sigma_t^2  S^2}))$ \\  
  \nonl \textbf{// local gradient descent.} \\
     $W_i(t)_{l+1} \leftarrow W_i(t)_l - \eta_i \widetilde \nabla W_i(t)_l$  \\
        \nonl \textbf{// end $g_{til}$, and start next iteration.} \\
     } 
  \nonl  \textbf{// send out local updates.} \\
 $\Delta W_i(t) \leftarrow W_i(t)_{L} - W(t)$ // with $M$ layers \\ 
         \nonl \textbf{// end $f_{ti}$ for local parameter update.} \\
}  

 \vspace{0.2cm}
 \textbf{server do} \\
  \nonl \textbf{// collect local updates from $K_t$ clients.} \\
  $\Delta W_{i}(t), i=1,...K_t$ \\
  \nonl \textbf{// aggregation.}\\
 $W(t + 1) \leftarrow  W(t) + \frac{1}{K_t}\sum\nolimits_{i = 1}^{{K_t}} \Delta W_{i}(t) $\\
  \nonl  \textbf{// start next round of learning until reaching $T$.} \\
}
\textbf{Output:} global model $W(T)$
\end{algorithm}



\subsection{Fed-$\alpha$CDP: the General Algorithm}


\textbf{Algorithm~\ref{client_dpsgd_12252020-l2max}} provides the pseudo-code for the general Fed-$\alpha$CDP setup. As the sensitivity and noise scale jointly determine the noise variance,
there are several options for the choice of these parameters: fixed clipping based sensitivity $S=C$ or $l_2$-max sensitivity, and fixed noise scale $\sigma_t=\sigma_0$ or decaying noise scale. 
The Gaussian noise $\mathcal{N}(0,\sigma_t^2S^2)$ injected to per-example gradient is most adaptive to the training process under both the decaying $\sigma_t$ and the $l_2$-max sensitivity as the noise variance closely track the trend of the gradients throughout the learning process. It offers a larger noise at the early stage of training for gradient leakage resilience and a smaller noise at the latter stage for high accuracy.

\section{Privacy Analysis}
\label{sec:formal}

In this section, we demonstrate how Fed-$\alpha$CDP satisfies differential privacy. We first prove that the per-example gradient noise injection is differentially private. Then we show that the trained global model is differentially private due to the composition theorems (Theorem~\ref{theorem:composition_sequential} and~\ref{theorem:composition_parallel}) and Post-Processing (Theorem~\ref{theorem:post_processing}). In addition, we demonstrate that Fed-$\alpha$CDP with per-example gradient perturbation guarantees both per-example and per-client differential privacy. At last, we introduce the experimental privacy analysis approach employed to evaluate Fed-$\alpha$CDP.

To prove that the per-example gradient noise injection step is differentially private, we model the Gaussian noise density function as $\frac{1}{\zeta} e^{-\beta||\mu||}$ at the local level. Inspired by~\cite{chaudhuri2011differentially}, Theorem~\ref{theorem:adaptive} ensures the differential privacy guarantee of $l_2$-max sensitivity under certain forms of convexity.

\begin{theorem}
For each iteration with batch size $B$, adding Gaussian noise with density function $\frac{1}{\zeta} e^{-\beta||\mu||}$ to the training function with loss function $h(x,w(t))$ and $|\nabla h(x,w(t))|<S$ for sensitivity of the model at iteration $t$ ensures $\epsilon$ differential privacy of the training function if $\beta=\frac{B\epsilon}{2S}$.
\label{theorem:adaptive}
\end{theorem}
\begin{proof}
Let $R(x,w)$ and $r(x,w)$ be two continuous and differentiable functions.
When picking $\mu$ from the Gaussian distribution, for a specific $\mu_0$, the density at $\mu_0$ is proportional to $e^{-\beta||\mu_0||}$. Then we have for any $w$ given $x_1$ and $x_2$ be any two batches that differ in the value of one instance:
\begin{equation}
\frac{R(x_1,w)}{R(x_2,w)} = \frac{\frac{1}{\zeta} e^{-\beta||\mu_1||}}{\frac{1}{\zeta} e^{-\beta||\mu_2||}}   = e^{-\beta(||\mu_1||-||\mu_2||)}.
    \label{equa:adp1}
\end{equation}
Let $\omega_1$ and $\omega_2$ are the solutions, e.g. weights, respectively to non-private training function under neighboring inputs $x_1$ and $x_2$, then with triangle inequality, we have
\begin{equation}
||\mu_1||-||\mu_2|| \leq ||\mu_1-\mu_2||=||\omega_1-\omega_2||.
    \label{equa:adp2}
\end{equation}
Next we assume $R(w)$ and $r(w)$ are $\lambda$-strongly convex. By  $\lambda$-strongly convex~\cite{rockafellar2009variational}, we mean for all $\alpha \in (0,1)$, function $f_1$ and $f_2$ satisfy $ G( \alpha f_1 + (1-\alpha) f_2) \leq \alpha G(f_1) + (1-\alpha) G(f_2) - \frac{1}{2}\lambda\alpha(1-\alpha)||f_1-f_2||^2_2 $. Assign $R(w)$ to $h(x_1,w(t))$ and $r(w)$ to $h(x_2,w(t))-h(x_1,w(t))$, $\omega_1$ to $\arg \min_\omega h(x_1,w(t)$, and $\omega_2$ to $\arg \min_\omega h(x_2,w(t)$. According to the above assignment,
\begin{equation}
    \nabla R(\omega_1)=\nabla R(\omega_2) + \nabla r(\omega_2)=0
    \label{equa:adp3}
\end{equation}
Given the $\lambda$-strongly convex, we have, 
\begin{equation}
    (\nabla R(\omega_1) - \nabla R(\omega_1))^\top (\omega_1-\omega_2) \geq \lambda ||(\omega_1-\omega_2) ||^2.
    \label{equa:adp4}
\end{equation}
Combining Equation~\ref{equa:adp3} and~\ref{equa:adp4} and with Cauchy-Schwartz inequality, we get:
\begin{align*}
||\omega_1-\omega_2||*||\nabla r(\omega_2)|| \geq & (\omega_1-\omega_2)^T\nabla r(\omega_2)) \\
=     (\nabla R(\omega_1) - \nabla R(\omega_1))^T (\omega_1-\omega_2) \geq & \lambda ||(\omega_1-\omega_2) ||^2.
\end{align*}
By setting $\lambda$ to 1, we have $||\omega_1-\omega_2|| \leq \max_\omega ||\nabla r(\omega) || $ and $\nabla r(\omega) = \frac{1}{B} (x_1 \nabla h(x_1,w(t)) - x_2 \nabla h(x_2,w(t)))  \leq \frac{2S}{B}$.

Recall equation~\ref{equa:adp1} and equation~\ref{equa:adp2},  we have $\frac{R(x_1,w)}{R(x_2,w)} = e^{-\beta(||\mu_1||-||\mu_2||)} \leq e^{-\beta(||\omega_1 - \omega_2||)} \leq e^{-\beta\frac{2S}{B}}$. Combined with equation~\ref{equa:dp} in differential privacy definition, when $\beta=\frac{B\epsilon}{2S}$,  we have $\frac{R(x_1,w)}{R(x_2,w)} \leq \epsilon$ and thus completes the proof.
\end{proof}

{\bf Privacy Accounting.\/} 
Given that each local SGD is differentially private, next we accumulate the privacy spending with privacy composition in federated learning.
The base composition originated from the composition theorem is considered loose for accurately tracking per-step privacy spending.
Several privacy accounting methods have been proposed for tight privacy composition.
The most representative ones are advanced composition~\cite{dwork2010boosting}, zCDP~\cite{yu2019differentially}, and Moments Accountant~\cite{abadi2016deep}. All of them commonly assume a fixed $\delta$ to track $\epsilon$. 
We will use Moments accountant in our evaluation for its being the state-of-the-art tight privacy accounting method.
Based on the default implementation\footnote{\url{https://github.com/tensorflow/privacy/blob/master/tensorflow_privacy/privacy/analysis/compute_dp_sgd_privacy.py}}, the privacy spending $\epsilon$ is computed given the total rounds $T$, the noise scale $\sigma$, the privacy parameter $\delta$, and the sampling rate $q$. For  data sampling in federated learning, we follow Fed-CDP~\cite{wei2021gradient} due to the resemblance in the client-side per-example noise injection process. Then, the local sampling with replacement over the disjoint local datasets across all clients in federated learning can be modeled as global sampling with replacement over the global data. Consequently, differentially private federated learning can be viewed as an alternative implementation of centralized deep learning with differential privacy in algorithmic logic.

\textbf{Per-example Differential Privacy.}
Given that Fed-$\alpha$CDP algorithms add differential privacy noise to per-example gradients generated for iteration $l$ at client $i$ during local model training ($1\leq l\leq L, 1\leq i\leq K_t$),
the sampling rate can be computed as follows. At each round $t$, $K_t$ clients are sampled with $\frac{K_t\times D_i}{|D|}$ and for each client $i$ out of $K_t$, a batch of size $B$ is sampled over $D_i$ ($\frac{B}{|D_i|}$) for the $L$ iterations. According to~\cite{wei2021gradient}, the sampled local data over $K_t$ clients at round $t$ can be viewed as a global data sampling over $D$ with the sampling rate $q_1=\frac{B\times K_t}{|D|}$, with the sampling size $B*K_t$. 
Thus, the locally added per-example noise would have a global instance-level effect over the global collection $D$ of distributed local data $D_i$ across the $N$ clients ($D=\cup_{i}^{N} D_i$).

\textbf{Per-Client Differential Privacy.} 
Fed-$\alpha$CDP offers both per-instance and per-client differential privacy guarantee. Although Fed-$\alpha$CDP uses the per-iteration local training ($\epsilon_{til}, \delta_{til}$)-differentially private function $g_{til}$ to compute and inject Gaussian noise to the per-example gradient during local model training, Fed-$\alpha$CDP ensures that the per-client local model parameter update is generated all over noisy gradients from $L$ local iterations. Therefore, the local parameter update shared at client $i$ by $f_{ti}$ is sanitized and $f_{ti}$ is ($\epsilon_{ti}, \delta_{ti}$) differentially privacy with  $\epsilon_{ti}=\sum_l^L \epsilon_{til}$ and $\delta_{ti}=\sum_l^L \delta_{til}$ by sequential composition. In comparison, Fed-SDP offers differential privacy only at the client level. This is because Fed-SDP 
only injects differential privacy guided Gaussian noise to per-client local model update after completing local training at the client. Therefore, the client-level privacy for Fed-SDP relies on the client sampling
rate of $q_2 = \frac{K_t}{N}$ and per-round composition~\cite{mcmahan2017learning,geyer2017differentially}.  

\begin{table}[t]
\centering
\caption{\small Benchmark datasets and parameters}
\vspace{-0.1cm}
\scalebox{0.80}{
\small{
\begin{tabular}{|c|c|c|c|c|c|}
\hline
                  & MNIST  & CIFAR10 & LFW     & Adult  & Cancer \\ \hline
\# training data      & 60000  & 50000   & 2267    & 36631  & 426    \\ \hline
\# validation data       & 10000  & 10000   & 756     & 12211  & 143    \\ \hline
\# features        & 28*28  & 32*32*3 & 32*32*3 & 105    & 30     \\ \hline
\# classes         & 10     & 10      & 62      & 2      & 2      \\ \hline
\# data/client     & 500    & 400     & 300     & 300    & 400    \\ \hline
\# local iteration $L$ & 100    & 100     & 100     & 100    & 100    \\ \hline
local batch size $B$   & 5      & 4       & 3       & 3      & 4      \\ \hline
\# rounds  $T$        & 100    & 100     & 60      & 10     & 3      \\ \hline
non-private acc    & 0.9798 & 0.674   & 0.695   & 0.8424 & 0.993  \\ \hline
non-private cost(ms)    & 6.8 & 32.5   & 30.9   & 5.1 & 4.9  \\ \hline

\end{tabular}
}}
\label{table:dataset_setup}
\vspace{-0.4cm}
\end{table}

When conducting privacy analysis for a family of $(\epsilon,\delta)$-differentially private federated learning algorithms, two independent yet complementary methods can be employed~\cite{wei2021gradient_tifs}. (i) We can set a fixed target privacy budget and measure the accumulated privacy spending (or privacy loss) $\epsilon$ at each step. The training will stop when the accumulation of the per-step $\epsilon$ exceeds the target privacy budget at step $t$. In this {\it target privacy budget} approach, we evaluate and compare the outcome of different differential privacy algorithms in terms of the achieved model accuracy under the given target privacy budget. (ii) We can instead specify the target utility (model accuracy goal) and the total number of steps $T$. Then we measure the accumulated privacy spending at the final step $T$ or any $t<T$ when the target utility is met. In this {\it target model accuracy} approach, we analyze and compare the privacy spending of different differential privacy algorithms in terms of accumulated privacy spending for achieving the target accuracy (utility). Given that different differential privacy algorithms with different settings of parameters $S$ and $\sigma$ may have different convergence speeds under the given target accuracy, and some may achieve high accuracy sooner and thus terminate early, this may further impact the total accumulated per-step privacy spending.

 \begin{figure*}
\centerline{\includegraphics[scale=.50]{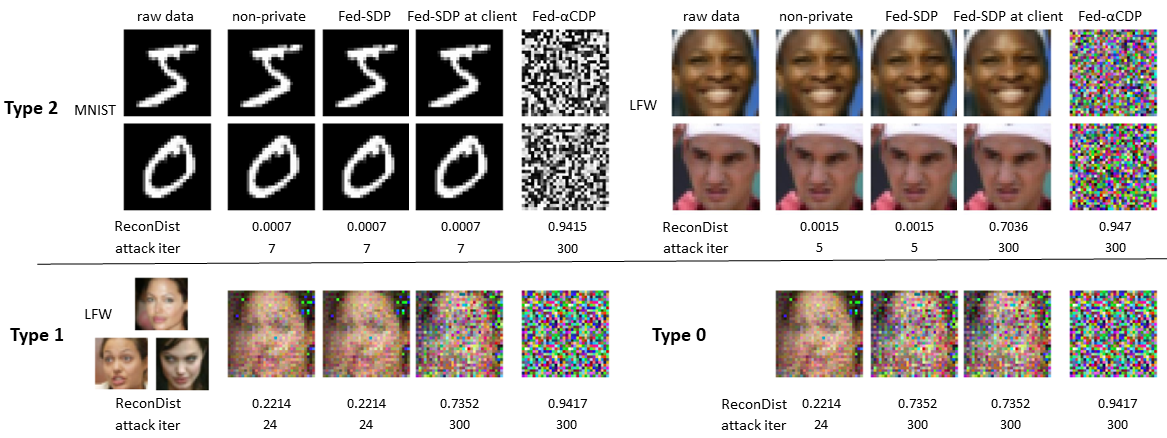}}
\vspace{-0.1cm}
\caption{\small Gradient leakage attack evaluation. With the maximum attack iteration set to 300. Fed-SDP can defend type-0 leakage but not type-1 leakage since its differential privacy sanitation is performed at the server.}
\label{figure:attack_evaluation}
\vspace{-0.2cm}
\end{figure*}

\section{Experimental Evaluation}
\label{sec:experiments}


We evaluate the proposed Fed-$\alpha$CDP on five benchmark datasets: MNIST, CIFAR10, LFW, Adult, and Cancer.
\textbf{Table~\ref{table:dataset_setup}} provides a detailed description of each dataset given the parameter setup. Specifically for each dataset, we include the training data/validation data split, number of features, number of classes, amount of data per client, number of local iterations batch size at local training, number of global rounds, as well as the test accuracy, and training cost for the non-private federated learning model. 
Similar to existing deep learning and federated with differential privacy literature~\cite{abadi2016deep,yu2019differentially,geyer2017differentially}, we evaluate the three image benchmark datasets on a deep convolutional neural network with two convolutional layers and one fully-connected layer. 
For attribute datasets, a fully-connected model with two hidden layers is used. 
Our federated learning setup follows the simulator in~\cite{geyer2017differentially} with a total of $N$ clients varying from 100, 1000 to 10000, and $K_t$ is set at a varying percentage of $N$, e.g., 5\%, 10\%, 20\%, 50\% per round. Unless otherwise specified, we use  $N=1000$ clients and 10\% of $N$ per round, e.g., $K_t=100$ participating clients. For privacy parameters, we follow~\cite{abadi2016deep} and set $C=4$ and $\sigma=6$ as default.
All experiments are conducted on an Intel 4 core i5-7200U CPU@2.50GHz machine with an NVIDIA Geforce 2080Ti GPU. 
We first report the empirical results on the resilience of Fed-$\alpha$CDP in comparison to existing differential privacy solutions. Second, 
we show that Fed-$\alpha$CDP obtains stronger differential privacy guarantee in terms of smaller $\epsilon$ spending under a target accuracy. Meanwhile, Fed-$\alpha$CDP can significantly boost the accuracy performance compared to Fed-SDP under a target privacy budget.
Finally, we show that besides the stronger leakage resilience, the proposed Fed-$\alpha$CDP outperforms existing adaptive clipping approaches: the quantile-based clipping~\cite{thakkar2019differentially}\footnote{\url{https://github.com/google-research/federated/tree/master/differential_privacy}} and AdaClip~\cite{pichapati2019adaclip}\footnote{\url{https://github.com/soominkwon/DP-dSNE}} with better test accuracy. We adopt their clipping implementation to our federated learning setup for fair comparison.

\begin{table}[t]
\centering
\caption{\small Average attack effectiveness comparison from 100 clients. $\tau$ is defined as the attack termination condition: the round when the constructed image $x_{rec}$ is visually similar to the private training data $x$ or 300 if the attack cannot succeed beforehand.}
\vspace{-0.1cm}
\scalebox{0.74}{
\small{
\begin{tabular}{|c|c|c|c|c|c|c|c|}
\hline
\multicolumn{2}{|c|}{\multirow{2}{*}{attack effect}}                      & \multicolumn{3}{c|}{MNIST} & \multicolumn{3}{c|}{LFW}   \\ \cline{3-8} 
\multicolumn{2}{|c|}{}                                       & resilient & ReconDist & $\tau$ & resilient & ReconDist & $\tau$ \\ \hline
\multirow{6}{*}{\rotatebox{90}{\scriptsize type 0}} & non-private                     & \textcolor{red}{N}       & 0.1549 & 6       & \textcolor{red}{N}       & 0.2214 & 24      \\ \cline{2-8} 
                           & Fed-SDP                  & Y       & 0.6991 & 300     & Y       & 0.7352 & 300    \\ \cline{2-8}
                    & Fed-SDP at client                  & Y       & 0.6991 & 300     & Y       & 0.7352 & 300    \\ \cline{2-8}
                      & Quantile~\cite{thakkar2019differentially}                         & \textcolor{red}{N}       & 0.4113 & 9     & \textcolor{red}{N}       & 0.3955 & 28     \\ \cline{2-8} 
                              & Adaclip~\cite{pichapati2019adaclip}                  & \textcolor{red}{N}       & 0.4647 & 9     & \textcolor{red}{N}       & 0.3421 &    27  \\ \cline{2-8} 
                           & {\bf Fed-$\alpha$CDP[$\sigma$]}  & Y       & \textbf{0.937}  & 300     & Y       & \textbf{0.9417} & 300     \\ \hline
\multirow{6}{*}{\rotatebox{90}{\scriptsize type 1}}  & non-private & \textcolor{red}{N}       & 0.1549 & 6       & \textcolor{red}{N}       & 0.2214 & 24   \\ \cline{2-8} 
  & Fed-SDP                  &    \textcolor{red}{N}    & 0.1549 & 6     & \textcolor{red}{N}    &   0.2214 & 24   \\ \cline{2-8} 
                           & Fed-SDP at client               & Y       & 0.6991 & 300     & Y       & 0.7352 & 300       \\ \cline{2-8} 
                             & Quantile~\cite{thakkar2019differentially}                         & \textcolor{red}{N}       & 0.5113 & 9     & \textcolor{red}{N}       & 0.3955 & 28     \\ \cline{2-8} 
                              & Adaclip~\cite{pichapati2019adaclip}                  & \textcolor{red}{N}       & 0.4647 & 9     & \textcolor{red}{N}       & 0.3421 &    27  \\ \cline{2-8} 
                           & {\bf Fed-$\alpha$CDP[$\sigma$]}  & Y       & \textbf{0.937}  & 300     & Y       & \textbf{0.9417} & 300      \\ \hline
\multirow{6}{*}{\rotatebox{90}{\scriptsize type 2}}    & non-private                     & \textcolor{red}{N}       & 0.0008 & 7       & \textcolor{red}{N}       & 0.0015 & 25      \\ \cline{2-8} 
                           & Fed-SDP                         & \textcolor{red}{N}       & 0.0008 & 7       & \textcolor{red}{N}       & 0.0015 & 25      \\ \cline{2-8} 
                             & Fed-SDP at client                           & \textcolor{red}{N}       & 0.0008 & 7       & \textcolor{red}{N}       & 0.0015 & 25      \\ \cline{2-8} 
                              & Quantile~\cite{thakkar2019differentially}                         & \textcolor{red}{N}       & 0.1252 & 9     & \textcolor{red}{N}       & 0.3955 & 28     \\ \cline{2-8} 
                              & Adaclip~\cite{pichapati2019adaclip}                  & \textcolor{red}{N}       & 0.2465 & 9     & \textcolor{red}{N}       & 0.3421 &    27  \\ \cline{2-8} 
                           &  {\bf Fed-$\alpha$CDP[$\sigma$]}  & Y       & \textbf{0.9415} & 300     & Y       & \textbf{0.947}  & 300     \\ \hline
\end{tabular}
}}
\label{table:attack_evaluation}
\vspace{-0.4cm}
\end{table}

\subsection{Gradient Leakage Resiliency}
\label{sec:leakage}

\begin{table*}[t]
\centering
\caption{\small Effectiveness of adaptive sensitivity optimization in model accuracy under 12 different settings of federated learning on MNIST with $T=100$ by varying $N$ clients and $K_t/N$, with $C=4$, $\sigma=6$. Time cost measured by second per local iteration.}
\vspace{-0.1cm}
\scalebox{0.85}{
\small{
\begin{tabular}{|c|c|c|c|c|c|c|c|c|c|c|c|c|c|}
\hline
                            & \multicolumn{4}{c|}{N=100}    & \multicolumn{4}{c|}{N=1000}   & \multicolumn{4}{c|}{N=10000}  & \multirow{2}{*}{cost} \\ \cline{1-13}
\%=K\_t/N                      & 5\%   & 10\%  & 20\%  & 50\%  & 5\%   & 10\%  & 20\%  & 50\%  & 5\%   & 10\%  & 20\%  & 50\%  &                       \\ \hline
non-private                 & 0.924 & 0.954 & 0.959 & 0.965 & 0.977 & 0.980 & 0.978 & 0.978 & 0.979 & 0.980 & 0.980 & 0.979 & 0.0068s               \\ \hline
Fed-SDP                     & 0.803 & 0.823 & 0.834 & 0.872 & 0.925 & 0.928 & 0.934 & 0.937 & 0.935 & 0.939 & 0.941 & 0.944 & 0.0069s               \\ \hline
Fed-CDP~\cite{wei2021gradient}                     & 0.815 & 0.831 & 0.858 & 0.903 & 0.951 & 0.956 & 0.96  & 0.964 & 0.966 & 0.963 & 0.968 & 0.966 & 0.0224s               \\ \hline
Fed-$\alpha$CDP             & \textbf{0.854} & \textbf{0.861} & \textbf{0.892} & \textbf{0.921} & \textbf{0.965} & \textbf{0.979} & \textbf{0.977} & \textbf{0.976} & \textbf{0.975} & \textbf{0.978} & \textbf{0.979} & \textbf{0.98}  & 0.023s                \\ \hline
\end{tabular}
}}
\label{table:kt_vary}
\vspace{-0.4cm}
\end{table*}


This section evaluates the resilience of Fed-$\alpha$CDP[$\sigma$] by comparing it with four most relevant and representative differential privacy techniques against all three types of gradient leakage threats (recall Threat Model in Section~\ref{sec:threatmodel}).   
\textbf{Table~\ref{table:attack_evaluation}} reports the measurement results from 100 clients in terms of the resilience (Yes or No), the average attack reconstruction distance, and attack iterations ($\tau$). We measure the attack reconstruction distance by the root mean square deviation (RMSE) between reconstructed input $x_{rec}$ and its private ground-truth counterpart $x$: $\frac{1}{U}\sum\nolimits_{i = 1}^U {(x(i) - {x_{rec}}(i)} {)^2}$ where $U$ denotes the total number of features in the input. 
In Table~\ref{table:attack_evaluation}, the column $\tau$ for MNIST (5th column) and the column $\tau$ for LFW (8th column) are reporting the number of reconstruction iterations that the gradient leakage attack can successfully reconstruct the training data $x_{rec}$, such that $x_{rec}$ and the private training data $x$ are visually indistinguishable. 
We set the attack termination condition to be 300 iterations in the sense that when reaching 300 iterations, if the reconstructed training data $x_{rec}$ from the leaked gradient is still not , visually close to the private training data $x$, the attack is considered failed. Therefore, the very small number of attack iterations, i.e., around 7 attack iterations for MNIST and 26 attack iterations for LFW, means that the reconstruction-based gradient leakage attack can succeed before reaching the maximum attack iterations of 300. 
The attack experiment is performed on gradients from the first local iteration as gradients at early training iterations tend to leak more information than gradients in the later stage~\cite{wei2020framework}. 
We make three observations: (1) 
Under type-0 attack (the top 6 rows), we show that Fed-SDP~\cite{mcmahan2017learning}, Fed-SDP at client~\cite{geyer2017differentially}, and our Fed-$\alpha$CDP are equally resilient, because attack is failed after exhausted all 300 attack iterations. Fed-$\alpha$CDP shows the highest RMSE value, indicating that the protection by Fed-$\alpha$CDP is strongest with the largest distance between the reconstructed training data $x_{rec}$  and the raw training data $x$. Type-1 gradient leakage attack and the attackers succeeded after 6 iterations for MNIST and 24 iterations for LFW. In comparison, Fed-SDP at client~\cite{geyer2017differentially} and our Fed-$\alpha$CDP remain to be resilient because attack is failed after 300 attack iterations. Similarly, Fed-$\alpha$CDP shows the highest RMSE value measured between training data $x_{rec}$  and the raw training data $x$, indicating that Fed-$\alpha$CDP provides stronger protection than that of~\cite{geyer2017differentially} (Fed-SDP at client). 
Finally, under type-2 attack (the bottom 6 rows), we show that both Fed-SDP~\cite{mcmahan2017learning} and Fed-SDP at client~\cite{geyer2017differentially} fail to protect against type-2 gradient leakage attack, and in both differentially private scenarios, the attacker succeeded after 6 iterations for MNIST and 25 iterations for LFW. In comparison, our Fed-$\alpha$CDP remain to be resilient with high RMSE. Note that we will use Fed-SDP in the rest of the paper for client-level differential privacy noise injection. Given that Fed-SDP and Fed-SDP at client apply the same noise at different locations, such noise injection procedure would only impact the capability of gradient leakage resilience rather than the accuracy and $\epsilon$-privacy spending.
(2) Fed-$\alpha$CDP[$\sigma$] outperforms the other four approaches with the highest level of resilience against gradient leakage attacks with the largest reconstruction distance, making the attack hardest to succeed for all three types of leakages.
(3) The quantile-based clipping~\cite{thakkar2019differentially} and AdaClip~\cite{pichapati2019adaclip} have poor overall resilience against all three types of gradient leakages. This is likely because both focused on reducing the clipping bound. With the sensitivity defined by the declining clipping bound for accuracy improvement, both reduce the amount of noise added as the number of training rounds increases. But both are much weaker in gradient leakage resilience due to insufficient differential privacy noise injection, especially in the early stage of training. 
\textbf{Figure~\ref{figure:attack_evaluation}} shows a visual comparison of Fed-$\alpha$CDP[$\sigma$] with Fed-SDP using examples from MNIST and LFW datasets, demonstrating the stronger resilience of Fed-$\alpha$CDP[$\sigma$] under all three types of gradient leakage attacks.



\subsection{$l_2$-max Sensitivity: Utility Impact Analysis}

\textbf{Impact of $N$ and $K_t$\/.} 
We first compare the accuracy performance of Fed-$\alpha$CDP with Fed-SDP, Fed-CDP baseline~\cite{wei2021gradient}, and the non-private case under 12 different federated learning configurations by varying total $N$ clients from 100, 1000 to 10000. In each setting of $N$, we additionally vary $K_t$ contributing clients at each round with 5\%, 10\%, 20\%, and 50\% of $N$.
\textbf{Table~\ref{table:kt_vary}} reports the results. We make two observations: (1) Fed-$\alpha$CDP consistently achieves the highest accuracy for all 12 settings of $N$ and $K_t$ and consistently improves the Fed-CDP baseline. With all other parameters fixed ($C, \sigma, T, K_t, N$), Fed-$\alpha$CDP benefits significantly from $l_2$-max sensitivity. (2) With $l_2$-max sensitivity, Fed-$\alpha$CDP can achieve an accuracy comparable to that of the non-private case, especially for $N=1000$ or larger. 
With $N=10000$, Fed-$\alpha$CDP can offer 97.5\% to 98.0\% accuracy performance for all four settings of $K_t$, compared to 97.9\% of the non-private algorithm.

\begin{table}[t]
\centering
\caption{\small  Impact of varying clipping bound settings on the accuracy performance of Fed-$\alpha$CDP, compared to Fed-SDP and Fed-CDP with fixed $S=C$ and fixed $\sigma$ for all five datasets ($\sigma=6$). The accuracy is measured at 100, 100, 60, 10, and 3 rounds for MNIST, CIFAR10, LFW, Adult, and Cancer datasets, respectively. For each dataset, the best accuracy performance is highlighted.}
\vspace{-0.1cm}
\scalebox{0.72}{
\small{
\begin{tabular}{|c|c|c|c|c|c|c|c|c|c|}
\hline
\multicolumn{2}{|c|}{clipping C}                                 & 0.1 & 0.5 & 1   & 2   & 4   & 8   & 16  & 32   \\ \hline
\multirow{4}{*}{\rotatebox{90}{\footnotesize MNIST}}    & non-private                 & \multicolumn{8}{c|}{0.980}                                      \\ \cline{2-10} 
                          & Fed-SDP                    & 0.883 & 0.909 & 0.944 & 0.940 & 0.928 & 0.911 & \textcolor{red}{0.656} & \textcolor{red}{0.214} \\ \cline{2-10} 
                          
                           & Fed-CDP~\cite{wei2021gradient}                    & 0.885 & 0.914 & 0.934 & 0.943 & \textbf{0.956} & 0.923 & 0.695 & 0.248 \\ \cline{2-10} 
                     
                          & Fed-$\alpha$CDP             & 0.886 & 0.917 & 0.945 & 0.965 & \textbf{0.979} & 0.975 & 0.976 & 0.976   \\ \hline
                      
\multirow{4}{*}{\rotatebox{90}{\footnotesize CIFAR10}}  & non-private                 & \multicolumn{8}{c|}{0.674}                                 \\ \cline{2-10} 
                        & Fed-SDP                     & \textcolor{red}{0.275} & \textcolor{red}{0.430} & 0.546 & 0.609 & 0.627 & \textcolor{red}{0.466} & \textcolor{red}{0.269} & \textcolor{red}{0.189}  \\ \cline{2-10} 
                       
                          & Fed-CDP~\cite{wei2021gradient}                     & 0.272 & 0.408 & 0.568 & 0.619 & \textbf{0.633} & 0.611 & 0.326 & 0.189  \\ \cline{2-10} 
                    
                          & Fed-$\alpha$CDP             & \textcolor{red}{0.286} & \textcolor{red}{0.439} & 0.576 & 0.628 & 0.645 & 0.643 & \textbf{0.649} & 0.647   \\ \hline
                   
\multirow{4}{*}{\rotatebox{90}{\footnotesize LFW}}     & non-private                 & \multicolumn{8}{c|}{0.695}                                     \\ \cline{2-10} 
                          & Fed-SDP                     & \textcolor{red}{0.248} & 0.592 & 0.601 & 0.635 & 0.635 & \textcolor{red}{0.534} & \textcolor{red}{0.257} &\textcolor{red}{0.117} \\ \cline{2-10} 
                        
                          & Fed-CDP~\cite{wei2021gradient}                      & 0.237 & 0.582 & 0.594 & 0.619 & \textbf{0.649} & 0.601 & 0.295 & 0.104 \\ \cline{2-10} 
                 
                          & Fed-$\alpha$CDP             & \textcolor{red}{0.251} & 0.594 & 0.603 & 0.642 & \textbf{0.683} & 0.671 & 0.668 & 0.654   \\ \hline
                    
\multirow{4}{*}{\rotatebox{90}{\footnotesize Adult}}     & non-private                 & \multicolumn{8}{c|}{0.843}             \\ \cline{2-10} 

  & Fed-SDP                    & \textcolor{red}{0.751} & 0.797  & 0.822 & 0.819 & 0.813 & 0.808 & \textcolor{red}{0.755} & \textcolor{red}{0.740}  \\ \cline{2-10} 
                          & Fed-CDP~\cite{wei2021gradient}                   & 0.752 & 0.81  & 0.822 & \textbf{0.825} & 0.824 & 0.761 & 0.757 & 0.755  \\ \cline{2-10} 
                          & Fed-$\alpha$CDP            & \textcolor{red}{0.755} & 0.825 & 0.832 & \textbf{0.838} & \textbf{0.838} & 0.83  & 0.831 & 0.834  \\ \hline
                
\multirow{4}{*}{\rotatebox{90}{\footnotesize Cancer}}    & non-private                 & \multicolumn{8}{c|}{0.993}                                     \\ \cline{2-10}

                         & Fed-SDP                     & 0.958 & 0.972 & 0.972 & 0.972 & 0.979 & 0.944 & 0.937 & \textcolor{red}{0.923}  \\ \cline{2-10} 
                         
                          & Fed-CDP~\cite{wei2021gradient}                   & 0.958 & 0.965 & 0.972 & \textbf{0.979} & \textbf{0.979} & 0.972 & 0.965 & 0.951  \\ \cline{2-10} 
                     
                          & Fed-$\alpha$CDP      & 0.965 & 0.972 & 0.979 & \textbf{0.993} & 0.986 & \textbf{0.993} & 0.986 & 0.986   \\ \hline
                
\end{tabular}
}}
\label{table:clipping_l2max}
\vspace{-0.4cm}
\end{table}


\textbf{Impact of varying clipping bound $C$.} 
This set of experiments evaluates the effectiveness of $l_2$-max sensitivity by comparing Fed-$\alpha$CDP with Fed-SDP, Fed-CDP baseline, and the non-private training for the five benchmark datasets with constant clipping varying $C$ from 0.1 to 32. \textbf{Table~\ref{table:clipping_l2max}} reports the results. We make two observations: (1) With $l_2$-max sensitivity, Fed-$\alpha$CDP consistently outperforms Fed-CDP and Fed-SDP under all settings of $C$ for all five datasets. (2) Fed-CDP with fixed clipping-based sensitivity performs poorly when the clipping bound $C$ is set to larger values (e.g., 16, 32) and smaller values (e.g., 0.1, 0.5, 1). This shows that Fed-CDP is overly sensitive to a proper setting of the clipping bound $C$ and requires careful tuning for finding the right clipping bound $C$ for different datasets to achieve optimal results. 
In comparison,
Fed-$\alpha$CDP is much less sensitive to the clipping bound setting, thanks to the adaptive nature of $l_2$-max sensitivity. Hence, Fed-$\alpha$CDP no longer depends on fine-tuning to find the proper clipping bound $C$ as long as the clipping bound $C$ is not too small, e.g., 1.5 or 2 times the $l_2$-max norm in non-private case. This is because when using a very small clipping bound, one may clip off too much information from those informative gradients prior to noise injection, especially in early rounds where per-example gradients tend to be large. 


\subsection{$l_2$-max Sensitivity: Privacy Impact Analysis}

\textbf{Privacy analysis under fixed total rounds $T$.}
In this set of experiments, we fix $\delta=1e-5$, $\sigma=6$,  $C=4$, and compare Fed-$\alpha$CDP with Fed-SDP in terms of privacy and utility (model accuracy). The $\epsilon$ spending is measured using Moments accountant, at rounds 100, 100, 60, 10, and 3 for MNIST, CIFAR10, LFW, Adult, and Cancer datasets.
\textbf{Table~\ref{table:privacy_evaluation}} reports the results. We make two observations: (1) Under fixed $\sigma$ and $T$, Fed-$\alpha$CDP and Fed-CDP will accumulate the same $\epsilon$ spending under all five privacy accounting methods. This is consistent with Lemma~\ref{lemma:gaussian_mechanism}: privacy spending $\epsilon$ is mainly correlated with noise scale $\sigma$ given a fixed $\delta$, a fixed sampling rate $q$ and a fixed $T$ rounds. The $\epsilon$ spending of Fed-$\alpha$CDP is 
smaller compared to Fed-SDP, indicating stronger differential privacy guarantee.
(2) Fed-$\alpha$CDP provides higher accuracy than that of Fed-SDP and Fed-CDP 1\% $\sim$ 5\% for all five datasets. This is because $l_2$-max sensitivity ensures smaller noise injection over the $T$ rounds of federated learning, whereas Fed-SDP and Fed-CDP use a constant noise variance due to fixed sensitivity and fixed $\sigma$, and inject the same amount of noise each time even when federated learning comes close to convergence.

\begin{table}[t]
\centering
\caption{\small Comparing Fed-SDP with Fed-$\alpha$CDP approaches for $\epsilon$ privacy spending using Moments Accountant with $C=4$, $\sigma=6$, $\delta=1e-5$. $\epsilon$ is measured at 100, 100, 60, 10, 3 rounds for MNIST, CIFAR10, LFW, Adult and Cancer datasets,  respectively.}
\vspace{-0.1cm}
\scalebox{0.80}{
\small{
\begin{tabular}{|c|c|c|c|c|c|c|}
\hline
\multicolumn{2}{|c|}{$\epsilon$ spending: fixed $T$} & MNIST & CIFAR10 & LFW   & adult & cancer \\ \hline
\multicolumn{2}{|c|}{non-private accuracy}   & 0.980 & 0.674   & 0.695 & 0.843 & 0.993  \\ \hline
\multirow{2}{*}{Fed-SDP}                & $\epsilon$ & 0.854 & 0.854   & 0.668 & 0.303 & 0.147  \\ \cline{2-7} 
                                        & accuracy        & 0.928 & 0.627   & 0.635 & 0.813 & 0.979  \\ \hline
\multirow{2}{*}{Fed-CDP~\cite{wei2021gradient}}                & $\epsilon$ & \textbf{0.823} & \textbf{0.823}   & \textbf{0.636} & \textbf{0.276} & \textbf{0.138}  \\ \cline{2-7} 
                                        & accuracy        & 0.956 & 0.633   & 0.649 & 0.824 & 0.979  \\ \hline
\multirow{2}{*}{Fed-$\alpha$CDP} & $\epsilon$ & \textbf{0.823} & \textbf{0.823}   & \textbf{0.636} & \textbf{0.276} & \textbf{0.138}  \\ \cline{2-7} 
                                        & accuracy  & \textbf{0.979} & \textbf{0.645}   & \textbf{0.683} & \textbf{0.838} & \textbf{0.986}  \\ \hline
\end{tabular}
}}
\label{table:privacy_evaluation} 
\vspace{-0.2cm}
\end{table}

\begin{table}[t]
\centering
\caption{\small Differential privacy spending $\epsilon$ for Fed-$\alpha$CDP and Fed-SDP under target accuracy goal with clipping bound $C=4$, fixed noise scale $\sigma=6$ and $\delta=1e-5$. The target accuracy is set to the accuracy of Fed-SDP for $T$ rounds given in Table~\ref{table:privacy_evaluation}.}
\vspace{-0.1cm}
\scalebox{0.85}{
\small{
\begin{tabular}{|c|c|c|c|c|c|c|}
\hline
\multicolumn{2}{|c|}{$\epsilon$ spending: target acc} & MNIST & CIFAR10 & LFW   & adult & cancer \\ \hline
\multicolumn{2}{|c|}{target acc}                      & 0.928 & 0.627   & 0.635 & 0.813 & 0.979  \\ \hline
\multirow{2}{*}{Fed-SDP}             & rounds         & 100   & 100     & 60    & 10    & 3      \\ \cline{2-7} 
                                     & $\epsilon$     & 0.854 & 0.854   & 0.668 & 0.303 & 0.147  \\ \hline
\multirow{2}{*}{Fed-$\alpha$CDP}     & rounds         & \textbf{11}    & \textbf{67}      & \textbf{34}    & \textbf{5}    & \textbf{1}      \\ \cline{2-7} 
                                     & $\epsilon$     & \textbf{0.285} & \textbf{0.672}   & \textbf{0.478} & \textbf{0.184} & \textbf{0.084}  \\ \hline
\end{tabular}
}}
\label{table:targetacc_changingrounds}
\vspace{-0.4cm}
\end{table}

\textbf{Privacy analysis under target accuracy.\/}
Recall Theorem~\ref{theorem:gaussian} and Moments accountant, all privacy accounting methods are dataset-independent and clipping-independent when the following parameters are fixed: noise scale $\sigma$, total $T$ rounds, sampling rate $q$, and privacy parameter $\delta$. Hence, in this set of experiments, we measure and compare differential privacy spending $\epsilon$ for Fed-$\alpha$CDP and Fed-SDP from a different perspective: we measure and compare their respective privacy spending under a target accuracy goal, given fixed noise scale $\sigma$. We set the target accuracy using the accuracy of Fed-SDP at round $T$ and measure the accumulated privacy spending and the number of rounds required for Fed-$\alpha$CDP to achieve the target accuracy for all five datasets. \textbf{Table~\ref{table:targetacc_changingrounds}} reports the results. We make two observations: (1) Fed-$\alpha$CDP achieves the target accuracy at the 11th, 67th, 34th, 5th, and 1st round for MNIST, CIFAR10, LFW, Adult, and Cancer datasets respectively. $l_2$-max sensitivity smoothly and adaptively reduces the amount of noise injected as federated learning progresses in the number of rounds. 
This result also
shows that $l_2$-max sensitivity alone can largely reduce the noise variance compared to fixed sensitivity using a constant clipping bound $C$. 
(2) By achieving the target accuracy early, prior to reaching the $T$ rounds, Fed-$\alpha$CDP offers a much smaller accumulated privacy spending $\epsilon$ compared to that of Fed-SDP. 

 \begin{table}[t]
\centering
\caption{\small Accuracy utility of Fed-$\alpha$CDP under four different noise scale decaying policies with $C=4$, $\delta=1e-5$ and under the target $\epsilon$ given in Table~\ref{table:privacy_evaluation}. $\sigma_{0}=15$. }
\vspace{-0.1cm}
\scalebox{0.83}{
\small{
\begin{tabular}{|c|c|c|c|c|c|c|}
\hline
\multicolumn{2}{|c|}{ accuracy utility: target $\epsilon$}                           & MNIST & CIFAR10 & LFW   & Adult & Cancer \\ \hline
\multicolumn{2}{|c|}{non-private accuracy}        & 0.980 & 0.674   & 0.695 & 0.843 & 0.993  \\ \hline
\multirow{3}{*}{Fed-CDP~\cite{wei2021gradient}} & target $\epsilon$   & 0.823 & 0.823   & 0.636 & 0.276 & 0.138  \\ \cline{2-7} 
& accuracy & 0.956 & 0.633   & 0.649 & 0.824 & 0.979  \\ \cline{2-7} 
                                      & rounds   & 100   & 100     & 60    & 10    & 3  
                                      \\ \hline
\multirow{2}{*}{\begin{tabular}[c]{@{}c@{}} Fed-$\alpha$CDP \\ fixed $\sigma$ \end{tabular}} 
& accuracy & 0.979 & 0.645   & 0.683 & 0.838 & 0.986  \\ \cline{2-7} 
                                      & rounds   & 100   & 100     & 60    & 10    & 3  
                                      \\ \hline
\multirow{2}{*}{\begin{tabular}[c]{@{}c@{}} Fed-$\alpha$CDP[$\sigma$] \\ linear $\sigma$ \end{tabular}}      & accuracy & 0.981 & 0.645   & 0.683 & 0.844 & 0.986  \\ \cline{2-7} 
                                      & rounds   & 95    & 95      & 55    & 8     & 3      \\ \hline
\multirow{2}{*}{\begin{tabular}[c]{@{}c@{}} Fed-$\alpha$CDP[$\sigma$] \\ staircase $\sigma$ \end{tabular}}   & accuracy & \textbf{0.982} & \textbf{0.646}   & \textbf{0.683} & \textbf{0.847} & \textbf{0.993}  \\ \cline{2-7} 
                                      & rounds   & 95    & 95      & 55    & 8     & 3      \\ \hline
\multirow{2}{*}{\begin{tabular}[c]{@{}c@{}} Fed-$\alpha$CDP[$\sigma$] \\ exponential $\sigma$ \end{tabular}}  & accuracy & \textbf{0.983} & \textbf{0.646}   & \textbf{0.684} & \textbf{0.847} & \textbf{0.993}  \\ \cline{2-7} 
                                      & rounds   & 96    & 96      & 55    & 8     & 3      \\ \hline
\multirow{2}{*}{\begin{tabular}[c]{@{}c@{}} Fed-$\alpha$CDP[$\sigma$] \\ cyclic $\sigma$ \end{tabular}}      & accuracy & 0.979 & 0.645   & 0.683 & 0.844 & 0.986  \\ \cline{2-7} 
                                      & rounds   & 80    & 80      & 47    & 7     & 3      \\ \hline
\end{tabular}
}}
\label{table:epsguidednoisedecay}
\vspace{-0.2cm}
\end{table}

\subsection{Dynamic Noise Scale: Privacy and Utility Analysis}
\label{sec:noise_scale_decay}

{\bf Utility analysis under target privacy budget $\epsilon$.\/}
We first evaluate the effectiveness of  Fed-$\alpha$CDP[$\sigma$] with noise scale decay when consuming the same amount of privacy budget $\epsilon$. For fair comparison, the target privacy budget $\epsilon$ is set to be the privacy spending of Fed-CDP at 100, 100, 60, 10, and 3 rounds for MNIST, CIFAR10, LFW, Adult, and Cancer datasets, respectively. All four noise scale decay policies are evaluated for adaptive $\sigma$ optimization with initial $\sigma_{0}=15$ and ending $\sigma_T$ set to 4.85. The dynamic noise scale $\sigma$ is set this way to not violate the $\epsilon$ setting in the original experiment for Fed-CDP for the same number of rounds. The early rounds hold a noise scale $\sigma_t$ much larger than the default $\sigma=6$ for gradient leakage resilience. Consequently, the resulting $\epsilon$ spending for Fed-$\alpha$CDP[$\sigma$] is smaller than the $\epsilon$ spending for Fed-CDP at this stage. The later rounds adopt a noise scale $\sigma_t$ smaller than the default $\sigma=6$ for accuracy. This leads to a larger $\epsilon$ spending of Fed-$\alpha$CDP[$\sigma$] than the $\epsilon$ spending of Fed-CDP at the later stage. Overall, we are able to achieve better accuracy with the same privacy spending at (approximately) the same round.
\textbf{Table~\ref{table:epsguidednoisedecay}} reports the results. 
We make the following two observations: 
(1)  under a target $\epsilon$ privacy spending and for all four decaying policies of adaptive $\sigma$, 
Fed-$\alpha$CDP[$\sigma$] with noise scale decaying consistently shows higher model accuracy and faster convergence compared to Fed-$\alpha$CDP with fixed noise scale and Fed-CDP baseline. When the accumulated privacy loss reaches the target privacy budget for all four algorithms, we terminate the global training and record the accuracy and number of rounds used. 
The result is consistent with our analysis: by combining $l_2$-max sensitivity and decaying noise scale, Fed-$\alpha$CDP[$\sigma$] will generate the most adaptive noise variance, resulting in injecting smoothly reduced noise as federated learning comes close to convergence. (2) Empirically, the adaptive noise scale $\sigma$ with exponential decay policy is the winner for Fed-$\alpha$CDP[$\sigma$], providing the smallest privacy loss measured by
$\epsilon$.

\subsection{Time Cost of Adaptive Parameter Optimization} 
\label{sec:timecost}
This set of experiments compares the time cost of Fed-$\alpha$CDP approaches with Fed-SDP, Fed-CDP, and the non-private algorithm on all five benchmark datasets. \textbf{Table~\ref{table:timecost}} reports the results. We make three observations. 
(1) The two Fed-$\alpha$CDP algorithms incur a very similar time cost over all five datasets, and only adds neglectable cost to the Fed-CDP baseline. 
(2) Since $K_t$ clients can perform local training in parallel in federated learning, the overall time spent for one round is determined by the slowest client local training at each round. Although Fed-$\alpha$CDP approaches incur additional computation overhead for generating the $l_2$-max sensitivity and sanitizing per-example gradients in each local iteration,  their absolute time cost remains low.
(3) The relative cost is smaller when the model is simpler with fewer parameters. For example, CIFAR10 and LFW spend relatively more time compared to MNIST. The two attribute datasets have a lower absolute time cost. 
It is also possible to combining differential privacy with gradient compression for reduced communication and sanitization cost, we consider it future work due to the space limit.

\begin{table}[t]
\centering
\caption{\small Time cost per local iteration per client in seconds}
\vspace{-0.1cm}
\scalebox{0.83}{
\small{
\begin{tabular}{|c|c|c|c|c|c|}
\hline
          & MNIST  & CIFAR10 & LFW    & Adult  & Cancer \\ \hline
non-private    & 0.0068 & 0.0325   & 0.0309 & 0.0051 & 0.0051 \\ \hline
Fed-SDP  & 0.0069 & 0.0338   & 0.0313 & 0.0052 & 0.0051 \\ \hline
Fed-CDP~\cite{wei2021gradient}  & 0.0224 & 0.1315   & 0.1124 & 0.0118 & 0.0119 \\ \hline
Fed-$\alpha$CDP  & 0.0230 & 0.1321 & 0.1146  &  0.0121 & 0.0120  \\ \hline
Fed-$\alpha$CDP[$\sigma$]  & 0.0231 & 0.1321 & 0.1147  &  0.0122 & 0.0120  \\ \hline

\end{tabular}
}}
\label{table:timecost}
\vspace{-0.2cm}
\end{table}

\subsection{Comparing Fed-$\alpha$CDP[$\sigma$] with Adaptive Clipping}


Adaptive clipping approaches~\cite{pichapati2019adaclip,thakkar2019differentially} can be viewed as an alternative mechanism to improve model accuracy of differentially private federated learning by reducing the noise injection over the iterative rounds of joint training.  Adaclip~\cite{pichapati2019adaclip} performs the clipping bound estimation based on the coordinates and the loss function, whereas  the quantile clipping~\cite{pichapati2019adaclip} estimates the clipping bound using the quantile of the unclipped gradient norm. 
In this set of experiments, we compare the effectiveness of Fed-$\alpha$CDP[$\sigma$] with Quantile clipping and Adaclip.
Given that both quantile clipping~\cite{pichapati2019adaclip} and Adaclip~\cite{pichapati2019adaclip} use a rather small noise scale ($\leq 1$) in their implementation, we adopt their clipping implementation into our default initial setting of $\sigma=6$ and $C_0=4$ for fair comparison and report their model accuracy.
\textbf{Table~\ref{table:adaptive_clipping}} shows the results. 
We make three observations: (1) Fed-$\alpha$CDP[$\sigma$] consistently outperforms all other approaches, achieving the largest accuracy enhancements. (2) Compared to AdaClip and Quantile clipping, Fed-$\alpha$CDP[$\sigma$] incurs the smallest time cost while offering the highest accuracy improvement. (3) Under the default setting, thw Quantile-based clipping and Adaclip are not resilient against gradient leakage attacks. In comparison, Fed-$\alpha$CDP offers high resilience with comparable model accuracy. 

\begin{table}[t]
\centering
\caption{\small Comparing with existing adaptive clipping methods in terms of accuracy, time cost (sec/local iteration), and attack resilience.}
\vspace{-0.1cm}
\scalebox{0.79}{
\small{
\begin{tabular}{|c|c|c|c|c|c|c|}
\hline
\multicolumn{2}{|c|}{}                        & MNIST  & CIFAR10 & LFW    & Adult  & Cancer \\ \hline
\multirow{3}{*}{non-private}  & accuracy   & \textbf{0.980}  & \textbf{0.674}   & \textbf{0.695}  & \textbf{0.843}  & \textbf{0.993}  \\ \cline{2-7} 
                                 & time cost  & 0.0068  & 0.0325  & 0.0309 & 0.0051 & 0.0051 \\ \cline{2-7} 
                                 & resilience & \multicolumn{5}{c|}{\textcolor{red}{N}}                      \\ \hline
\multirow{3}{*}{Quantile~\cite{thakkar2019differentially}}    & accuracy   & 0.976  & 0.631   & 0.659  & 0.836  & 0.986  \\ \cline{2-7} 
                                 & time cost  & 0.0546 & 0.2125  & 0.199  & 0.0192 & 0.0189 \\ \cline{2-7} 
                                 & resilience & \multicolumn{5}{c|}{\textcolor{red}{N}}                      \\ \hline
\multirow{3}{*}{Adaclip~\cite{pichapati2019adaclip}}         & accuracy   & 0.975  & 0.627   & 0.651  & 0.833  & 0.986  \\ \cline{2-7} 
                                 & time cost  & 0.0337 & 0.1884  & 0.1523 & 0.0156 & 0.0153 \\ \cline{2-7} 
                                 & resilience & \multicolumn{5}{c|}{\textcolor{red}{N}}                     \\ \hline
\multirow{3}{*}{Fed-$\alpha$CDP[$\sigma$]}  & accuracy   & \textbf{0.983}  & \textbf{0.646}   & \textbf{0.683}  & \textbf{0.847}  & \textbf{0.993}  \\ \cline{2-7} 
                                 & time cost  & 0.023  & 0.1321  & 0.1128 & 0.0123 & 0.0125 \\ \cline{2-7} 
                                 & resilience & \multicolumn{5}{c|}{Y}                      \\ \hline
\end{tabular}
}}
\label{table:adaptive_clipping}
\vspace{-0.4cm}
\end{table}

\subsection{Performance under different Privacy Accounting}

This set of experiments measures the accuracy and training rounds 
 of Fed-$\alpha$CDP under different privacy accounting methods: base composition~\cite{dwork2014algorithmic}, advanced composition~\cite{dwork2010boosting}, zCDP~\cite{yu2019differentially} and Moments accountant~\cite{abadi2016deep} for a given privacy budget. The set of experiments is conducted on MNIST with the default setting: $N=1000$, $K_t=100$, Clipping bound $C=4$, $\delta=1e-5$, $T=100$. As demonstrated in~\cite{abadi2016deep,yu2019differentially}, 
tracking privacy spending using different privacy accounting method
 is independent of the noise injection process. As shown in~\cite{yu2019differentially}, zCDP can be converted to standard differential privacy. Meanwhile, Moments accountant bridges R\'enyi differential privacy with standard differential privacy~\cite{mironov2017renyi}. Table~\ref{table:accounting} reports the results of different privacy accounting methods. We make two observations. (1) Under the same privacy spending, Fed-$\alpha$CDP with dynamic noise scale would achieve higher accuracy than the baseline approaches. Given the reduced noise due to adaptive noise scale, Fed-$\alpha$CDP[$\sigma$] is able to achieve accuracy comparable to the non-private setting with faster convergence demonstrated by smaller training rounds. (2) When different privacy accounting methods are applied to the same Gaussian noise, their measured accumulated privacy spending are different. For base composition of 100 rounds, the accumulated $\epsilon$ privacy spending is as large as 123.354, for advanced composition 7.450, for zCDP 1.159, and for Moments accountant 0.823. This shows that the privacy budget should be set based on the specific privacy accounting method. And the comparison of privacy spending also need to be under a given privacy accounting method.

\begin{table}[t]
\centering
\scalebox{0.77}{
\small{
\begin{tabular}{|c|c|c|c|c|c|}
\hline
\multicolumn{2}{|c|}{\multirow{2}{*}{target $\epsilon$ measuring acc}} & BaseComp & AdvComp  & zCDP   & MomentsAcc \\ \cline{3-6} 
\multicolumn{2}{|c|}{}                                                     & 123.354  & 7.450     & 1.159  & 0.823   \\ \hline
\multirow{2}{*}{Fed-CDP}       &      accuracy                       & \multicolumn{4}{c|}{0.956}                      \\ \cline{2-6} 
                                      & rounds                          & \multicolumn{4}{c|}{100}                        \\ \hline
\multirow{2}{*}{\begin{tabular}[c]{@{}c@{}}  Fed-$\alpha$CDP  \\ fixed $\sigma$\end{tabular}}     
                                      &     accuracy         & \multicolumn{4}{c|}{0.979}                      \\ \cline{2-6} 
                                      & rounds                          & \multicolumn{4}{c|}{100}                        \\ \hline
\multirow{2}{*}{\begin{tabular}[c]{@{}c@{}}Fed-$\alpha$CDP[$\sigma$] \\ linear $\sigma$ \end{tabular}}         
                                      & accuracy     & \textbf{0.984}    & \textbf{0.981}     & \textbf{0.981}   & \textbf{0.981}   \\ \cline{2-6} 
                                      & rounds                        & 97       & 95          & 95     & 95      \\ \hline
\multirow{2}{*}{\begin{tabular}[c]{@{}c@{}}Fed-$\alpha$CDP[$\sigma$] \\ staircase $\sigma$\end{tabular}}       
                                      &   accuracy   & \textbf{0.983}    & \textbf{0.982}    & \textbf{0.982}  & \textbf{0.982}   \\ \cline{2-6} 
                                      & rounds                         & 97       & 95         & 95     & 95      \\ \hline
\multirow{2}{*}{\begin{tabular}[c]{@{}c@{}} Fed-$\alpha$CDP[$\sigma$] \\ exponential $\sigma$\end{tabular}}     
                                      &  accuracy    & \textbf{0.986}    & \textbf{0.981}    & \textbf{0.983} & \textbf{0.983}    \\ \cline{2-6} 
                                      & rounds                        & 97       & 94         & 96     & 96      \\ \hline
\multirow{2}{*}{\begin{tabular}[c]{@{}c@{}}Fed-$\alpha$CDP[$\sigma$] \\ cyclic $\sigma$\end{tabular}}      
                                      & accuracy      &  \textbf{0.982}   & \textbf{0.979}   &  \textbf{0.981} & \textbf{0.979}    \\ \cline{2-6} 
                                      & rounds                         &     81   & 79       & 80     &    80   \\ \hline
\end{tabular}
}}
\vspace{-0.1cm}
\caption{\small Accuracy measurement of Fed-$\alpha$CDP under different privacy accounting methods: base composition~\cite{dwork2014algorithmic}, advanced composition~\cite{dwork2010boosting}, zCDP~\cite{yu2019differentially} and Moments accountant~\cite{abadi2016deep}. The set of experiments is conducted on MNIST with $N=1000$, $K_t=100$, Clipping bound $C=4$, $\delta=1e-5$, $T=100$. The non-private accuracy for MNIST is 0.980.
}
\label{table:accounting}
\vspace{-0.4cm}
\end{table}

\section{Related Work}
\label{sec:relatedwork}

\textbf{Privacy Leakages in Federated Learning.\/} Despite the default privacy in federated learning, simply keeping client training data local is insufficient for protecting the privacy of sensitive client training data due to gradient leakage attacks. 
The early attempt of gradient leakage~\cite{aono2017privacy} brought theoretical insights by showing provable reconstruction feasibility on a single neuron or single-layer networks in centralized deep learning. Due to less control on the client in the distributed learning, an unauthorized read of client local parameter updates leads to more attack surfaces for gradient leakage attacks~\cite{geiping2020inverting,zhu2019deep,zhao2020idlg,wei2020framework,wang2019beyond,zhu2020r,yin2021see},  
The privacy leakage attack discloses private client training data via unauthorized inference by gradient-based reconstruction. While we consider gradient leakage attacks using L-BFS as the optimizer similar to~\cite{zhu2019deep,zhao2020idlg}, with a small batch size (<16), and shallow models~\cite{wei2020framework}, recent attacks enhance the attack procedure with Adam optimizer for deeper machine learning models~\cite{geiping2020inverting}.  
By utilizing regularization terms for batch disaggregation~\cite{yin2021see,huang2021evaluating}, these attacks can
recover the private training data with a batch size over 30. 
Other known attacks in federated learning include membership inference~\cite{shokri2017membership,nasr2018comprehensive,truex2018towards},  
attribute inference~\cite{melis2019exploiting}, and model inversion~\cite{fredrikson2015model,jagielski2020high}, which can be launched at both client and federated server and cause adverse and detrimental effects when combined with gradient leakage attacks. 


\textbf{Privacy-enhanced Federated Learning.}  
Recent efforts can be classified into three broad categories: (1) {\it Differentially private federated learning}, in which several  proposals~\cite{mcmahan2017learning,geyer2017differentially,wei2020federated} have been put forward, most of which belong to the family of Fed-SDP and provide only client level differential privacy. 
Besides Gaussian mechanism,~\cite{agarwal2018cpsgd} injects differential privacy noise with Binomial mechanism.
\cite{wei2021gradient} has some preliminary studies on per-example per-iteration differential privacy noise for federated learning. All these methods suffer from the fixed privacy parameter setting as discussed in Section~\ref{sec:adaptive}. While many of the advancing techniques for differential privacy, including the two adaptive clipping approaches~\cite{pichapati2019adaclip,thakkar2019differentially} in our comparison study
are investigated in the centralized setting, novel differential privacy techniques need to be developed in the federated setting.
(2) {\it Federated learning with secure multiparty computation}, which incorporates secure multiparty computation (SMPC) to ensure secure communication of shared parameter updates between a client and the federated server~\cite{bonawitz2017practical}.
Therefore, similar to Fed-SDP, it can only secure sharing between a client and the federated server in distributed learning but cannot protect against any privacy attacks happened at the client (type 1 and type 2 gradient leakages) before the cryptography process.
 (3) {\it Federated learning with local differential privacy (LDP)}, which extends the LDP concept and techniques~\cite{kasiviswanathan2011can,gursoy2019secure} developed for protecting user identity when sharing web access log or click data. Although federated learning with LDP~\cite{truex2020ldp} perturbs the per-client gradient updates before sharing with the federated server and~\cite{sun2021ldpfl} locally splits and shuffles the weights before sharing, they do not generate a global model with differential privacy guarantee, and also cannot provide instance-level differential privacy for client training data against gradient leakage attacks. The $\epsilon$ in LDP is defined over the local data record rather than the global data space. 
Homomorphic encryption~\cite{paillier1999public} and trusted execution environments (secure enclaves)~\cite{mo2021ppfl} are alternative solutions to secure federated aggregation against malicious compromises. 
However, their protection will only be effecitve after the data is encrypted or inside the enclave. Hence they can be resilient against the type 0 and type 1 gradient leakages but are vulnerable to the type 2 gradient leakages.




\section{Conclusion}
\label{sec:conclusion}

We have presented the development of Fed-$\alpha$CDP, a gradient leakage resilient approach to securing distributed SGD in federated learning. Fed-$\alpha$CDP elevates the level of resilience against gradient leakages by incorporating per-example differential privacy, $l_2$-max sensitivity, and dynamic decaying noise scale. Instead of using constant noise variance and fixed noise to perturb the local gradient updates each time, these optimizations enable non-uniform noise variance with decreasing trend so that large noise is added in early rounds and smaller noise is used in later rounds of federated learning. A formal analysis is provided on the privacy guarantee of Fed-$\alpha$CDP. Extensive experiments are conducted over five benchmark datasets. We demonstrate that the proposed Fed-$\alpha$CDP approach outperforms the existing state-of-the-art in differentially private federated learning approaches with competitive accuracy performance, strong differential privacy guarantee, and high resilience against gradient leakage attacks.






%



\vspace{0.6cm}

\noindent \textbf{Acknowledgement.}
The authors acknowledge partial support by the National Science Foundation under Grants NSF 2038029, NSF 1564097, a CISCO grant on edge computing, and an IBM faculty award. 

\bibliographystyle{IEEEtran}
\bibliography{bare_jrnl_compsoc}

\vspace{-15pt} 

\begin{IEEEbiography}
[{\includegraphics[width=1in,height=1.25in,clip,keepaspectratio]{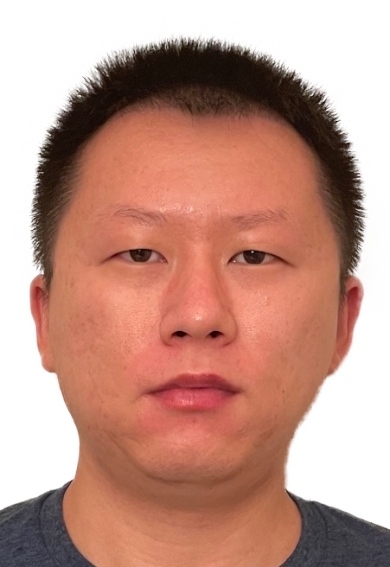}}]
{Wenqi Wei} 
is currently a tenure-track assistant professor in the Computer and Information Sciences Department, Fordham University. He 
obtained his PhD in the School of Computer Science, Georgia Institute of Technology, 
and received his B.E. degree from the School of Electronic Information
and Communications, Huazhong University of Science
and Technology. His research interests include data privacy, machine learning security, and big data analytics. 
\end{IEEEbiography}

\vspace{-15pt}

\begin{IEEEbiography}
[{\includegraphics[width=1in,height=1.25in,clip,keepaspectratio]{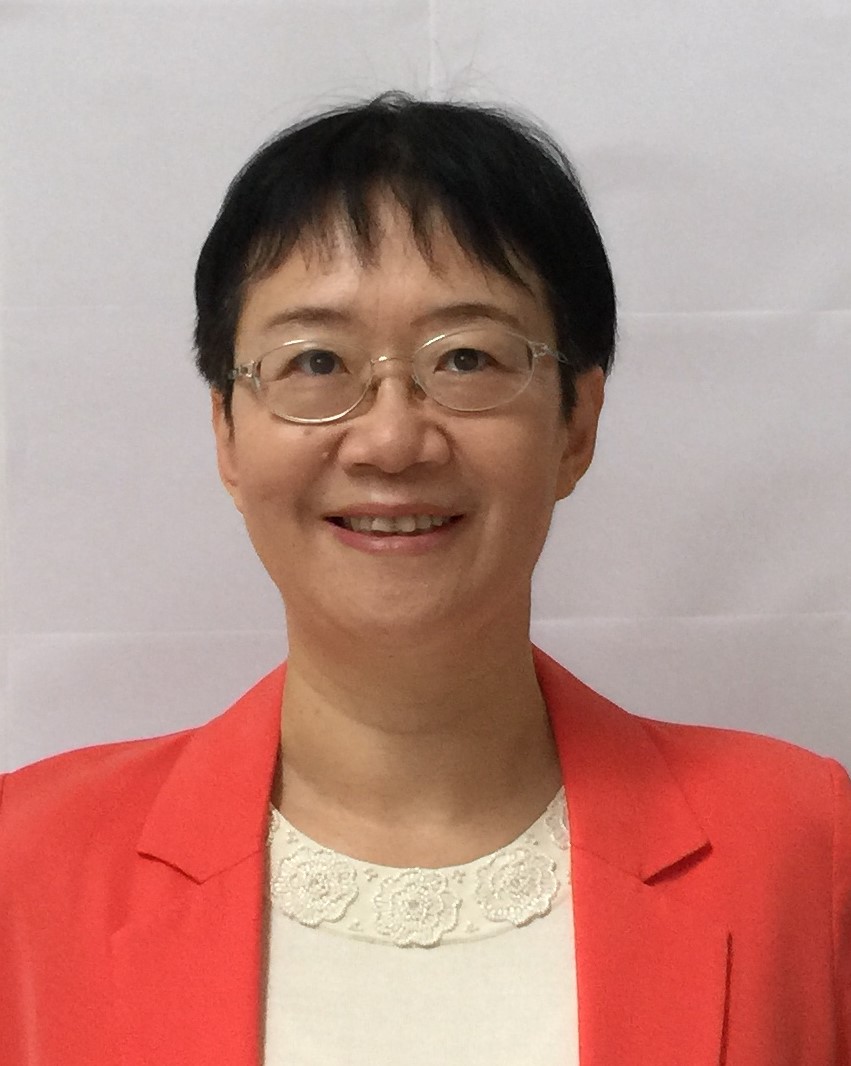}}]
{Ling Liu} 
is a professor in the School of Computer Science, Georgia Institute of Technology. She directs the research programs in the Distributed Data 
Intensive Systems Lab (DiSL). She is an elected IEEE Fellow, a recipient of the IEEE Computer Society Technical Achievement Award in 2012 and a recipient of the best paper award from a dozen of top venues, including ICDCS 2003, WWW 2004, 2005 Pat Goldberg Memorial Best Paper Award, IEEE Cloud 2012, IEEE ICWS 2013, ACM/IEEE CCGrid 2015, and IEEE Symposium on BigData 2016. 
Her current research is primarily sponsored by NSF, IBM, and Intel. 
\end{IEEEbiography}

\vspace{-15pt} 

\begin{IEEEbiography}
[{\includegraphics[width=1in,height=1.25in,clip,keepaspectratio]{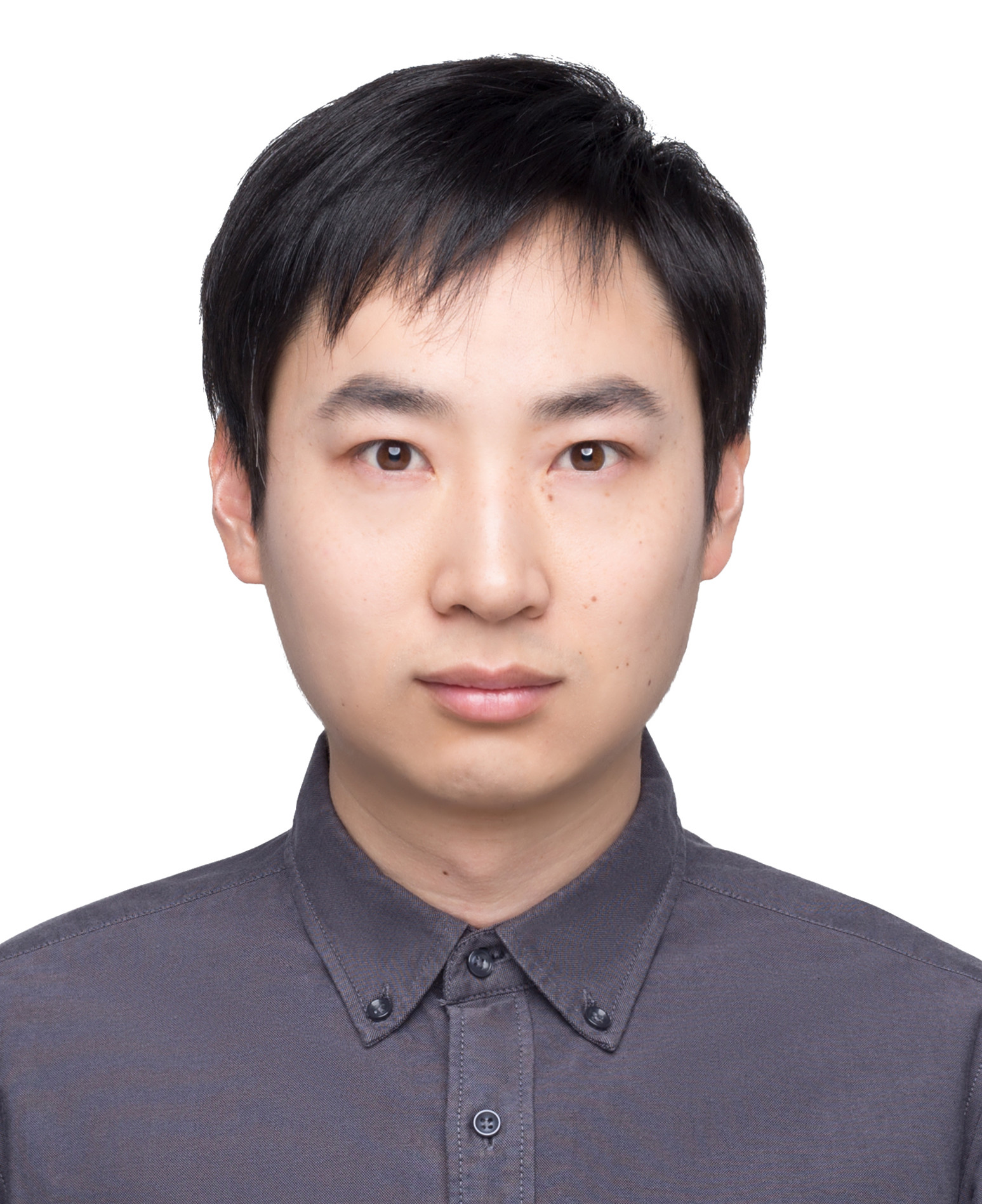}}]
{Jingya Zhou} received the BE degree in computer science from Anhui Normal University, Wuhu, in 2005, and his PhD degree in computer science from Southeast University, Nanjing, in 2013. He is currently an associate professor with the School of Computer Science and Technology, Soochow University, Suzhou, China. His research interests include cloud and edge computing, network embedding, data mining, online social networks and recommender systems.
\end{IEEEbiography}



\vspace{-15pt}



\begin{IEEEbiography}
[{\includegraphics[width=1in,height=1.25in,clip,keepaspectratio]{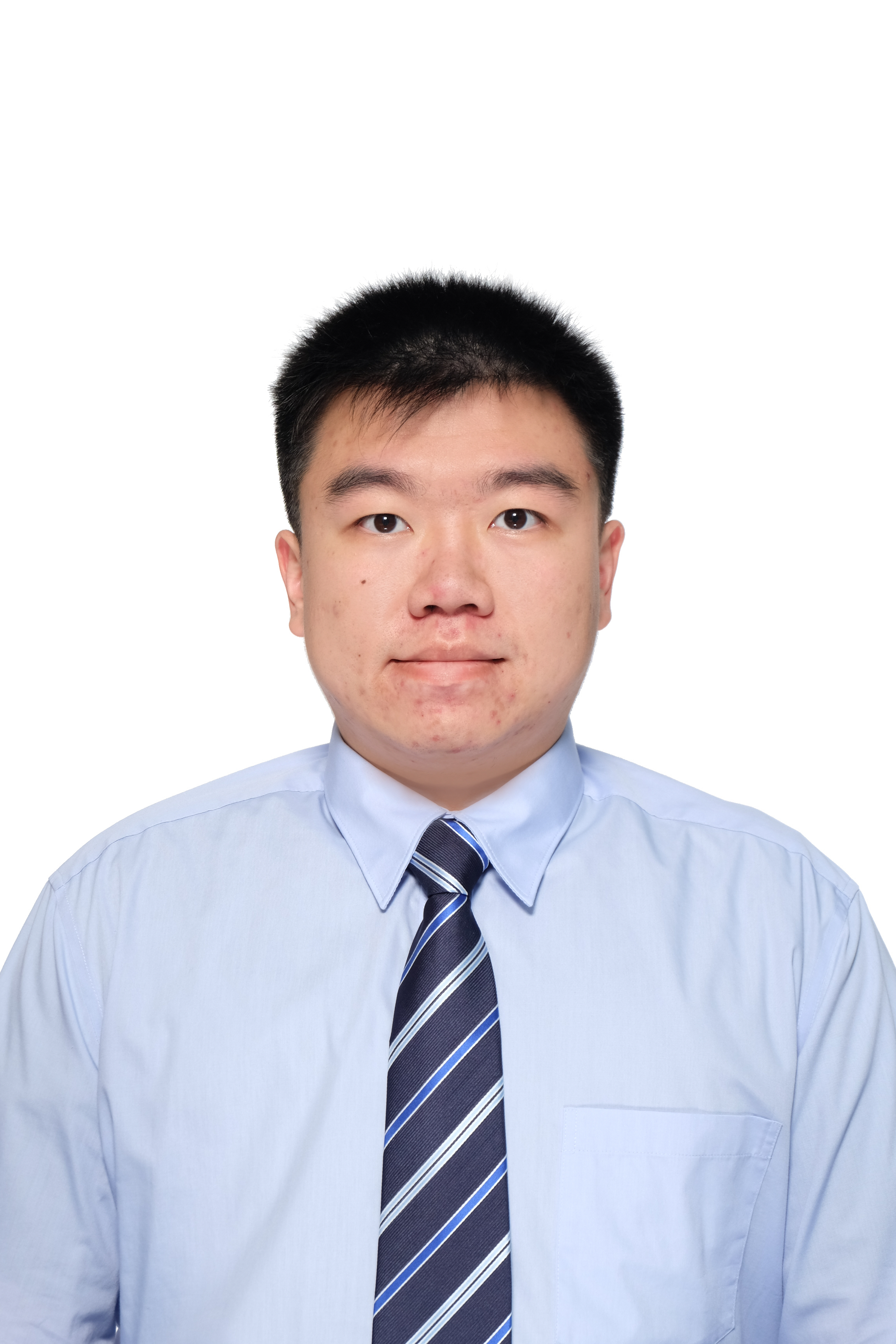}}]
{Ka-Ho Chow} is a Ph.D. candidate in the School of Computer Science at Georgia Institute of Technology, 
and received the BEng and MPhil degrees in computer science from the Hong Kong University of Science and Technology. His research interests include robust machine learning, cybersecurity, ML for systems, and mobile computing.
\end{IEEEbiography} 

\vspace{-15pt} 

\begin{IEEEbiography}
[{\includegraphics[width=1in,height=1.25in,clip,keepaspectratio]{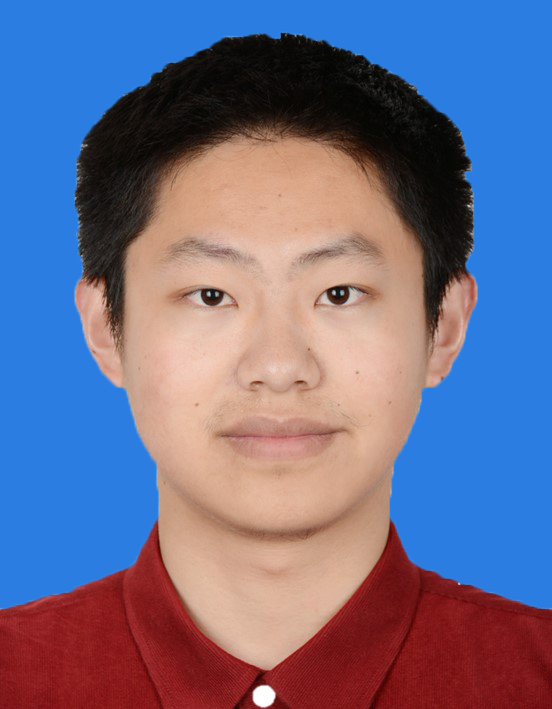}}]
{Yanzhao Wu} is currently a tenure-track assistant professor in the School of Computing and Information Sciences,
Florida International University. He 
obtained his PhD in the School of Computer Science, Georgia Institute of Technology, 
and his B.E. degree from the School of Computer Science and Technology, University of Science and Technology of China. His research interests primarily centered on systems for machine learning and big data, machine learning algorithms for optimizing computing systems, and edge AI.
\end{IEEEbiography}

\end{document}